%% file: ROT_AISTATS.tex
\documentclass[twoside]{article}

\usepackage[accepted]{aistats2024}

\usepackage[round]{natbib}

\usepackage[utf8]{inputenc} 
\usepackage[T1]{fontenc}    
\usepackage{url}            
\usepackage{booktabs}       
\usepackage{amsfonts}       
\usepackage{nicefrac}       
\usepackage{microtype}      
\usepackage{xcolor}         
\usepackage{resizegather}
\usepackage{mdframed}
\usepackage[inline]{enumitem}

\usepackage[pagebackref]{hyperref}
\hypersetup{
  colorlinks=true,
  linkcolor=blue,
  filecolor=magenta,
  urlcolor=cyan,
  citecolor=blue
}

\usepackage{url}            
\usepackage{booktabs}       
\usepackage{amsfonts}       
\usepackage{nicefrac}       
\usepackage{microtype}      
\usepackage{wrapfig}
\usepackage{amsmath,amsthm,amssymb,mathtools,enumitem,hyphenat,float}
\usepackage{bbm}
\usepackage{subcaption}
\usepackage{accents}
\usepackage{nccmath}
\usepackage{mleftright}
\usepackage{import}
\renewcommand{\epsilon}{\varepsilon}

\include{header}

\begin{document}

%

%

\twocolumn[

\aistatstitle{Structured Transforms Across Spaces \\ with Cost-Regularized Optimal Transport}

\aistatsauthor{ Othmane Sebbouh \And Marco Cuturi \And  Gabriel Peyr\'e }

\aistatsaddress{ ENS, CNRS, ENSAE \And  Apple \And ENS and CNRS } ]

\begin{abstract}
Matching a source to a target probability measure is often solved by instantiating a \textit{linear} optimal transport (OT) problem, parameterized by a ground cost function that quantifies discrepancy between points. 
When these measures live in the same metric space, the ground cost often defaults to its distance. When instantiated across two different spaces, however, choosing that cost in the absence of aligned data is a conundrum. As a result, practitioners often resort to solving instead a quadratic Gromow-Wasserstein (GW) problem.
We exploit in this work a parallel between GW and \textit{cost-regularized} OT, the regularized minimization of a linear OT objective parameterized by a ground cost.
    We use this cost-regularized formulation to match measures across two different Euclidean spaces, where the cost is evaluated between transformed source points and target points. We show that several quadratic OT problems fall in this category, and consider enforcing \textit{structure} in \textit{linear} transform (e.g. sparsity), by introducing structure-inducing regularizers. We provide a proximal algorithm to extract such transforms from unaligned data, and demonstrate its applicability to single-cell spatial transcriptomics/multiomics matching tasks.
\end{abstract}

\input{sections/intro.tex}
\input{sections/background.tex}
\input{sections/method.tex}

\input{sections/maps.tex}
\input{sections/applications.tex}

\bibliographystyle{apalike}
\bibliography{biblio}

\medskip

\clearpage

\appendix

\onecolumn 

\newpage

\input{sections/appendix.tex}

\end{document}

%% file: header.tex
\usepackage{wrapfig}

\usepackage{algorithm,algpseudocode}
\algdef{SE}[SUBALG]{Indent}{EndIndent}{}{\algorithmicend\ }%
\algtext*{Indent}
\algtext*{EndIndent}
\algblockdefx{Sync}{EndSync}[1]{\textbf{sync} #1}{\textbf{end sync}}
\usepackage{appendix}

\usepackage{bmpsize}
\usepackage[dvipdfmx]{graphicx}

\usepackage{amsmath}
\usepackage{cleveref}
\crefformat{section}{\S#2#1#3} 
\crefformat{subsection}{\S#2#1#3}
\crefformat{subsubsection}{\S#2#1#3}
\usepackage{autonum}
\usepackage{yhmath}
\usepackage{xcolor}
\usepackage{xfrac}
\usepackage{bm}
\usepackage{scalerel}
\usepackage{lipsum}
\usepackage[labelfont=bf, skip=0pt, font=footnotesize]{caption}

\newcommand\br[1]{\left ( #1 \right )}

\newcommand{\R}{\mathbb{R}}

\newcommand{\mA}{\mathbf{A}}
\newcommand{\mB}{\mathbf{B}}
\newcommand{\mC}{\mathbf{C}}

\newcommand{\mM}{\mathbf{M}}
\newcommand{\mU}{\mathbf{U}}
\newcommand{\mV}{\mathbf{V}}
\newcommand{\mSigma}{\boldsymbol{\Sigma}}

\newcommand{\mpi}{\boldsymbol{\pi}}

\newcommand{\m}[1]{\mathbf{#1}}

\newcommand{\cC}{\mathcal{C}}

\newcommand{\cP}{\mathcal{P}}

\newcommand{\cL}{\mathcal{L}}

\newcommand{\cO}{\mathcal{O}}
\newcommand{\cQ}{\mathcal{Q}}
\newcommand{\cR}{\mathcal{R}}

\newcommand{\cW}{\mathcal{W}}
\newcommand{\cX}{\mathcal{X}}
\newcommand{\cY}{\mathcal{Y}}
\newcommand{\cZ}{\mathcal{Z}}

\newcommand{\balpha}{\bm{\alpha}}
\newcommand{\bbeta}{\bm{\beta}}

\DeclareMathOperator*{\argmin}{\arg\!\min}

\newcommand{\sgn}{\mathrm{ \ sign}}

\newcommand*\dif{\mathop{}\!\mathrm{d}}

\newcommand{\norm}[1]{\left\lVert#1\right\rVert}

\newcommand{\sqn}[1]{{\left\lVert#1\right\rVert}^2}

\newcommand{\abs}[1]{\left\lvert#1\right\rvert}

\newcommand{\rk}{\mathrm{rk \ }}

\newcommand{\prox}{\mathrm{prox}}

\newcommand{\ip}[2]{\langle #1, #2 \rangle}

\newcommand*\diff{\mathop{}\!\mathrm{d}}
\makeatletter
\newsavebox\myboxA
\newsavebox\myboxB
\newlength\mylenA

\usepackage{tcolorbox}
\usepackage{pifont}
\definecolor{mydarkgreen}{RGB}{39,130,67}

\definecolor{mydarkred}{RGB}{192,47,25}

\newcommand*\overbar[2][0.75]{%
    \sbox{\myboxA}{$\m@th#2$}%
    \setbox\myboxB\null
    \ht\myboxB=\ht\myboxA%
    \dp\myboxB=\dp\myboxA%
    \wd\myboxB=#1\wd\myboxA
    \sbox\myboxB{$\m@th\overline{\copy\myboxB}$}
    \setlength\mylenA{\the\wd\myboxA}
    \addtolength\mylenA{-\the\wd\myboxB}%
    \ifdim\wd\myboxB<\wd\myboxA%
       \rlap{\hskip 0.5\mylenA\usebox\myboxB}{\usebox\myboxA}%
    \else
        \hskip -0.5\mylenA\rlap{\usebox\myboxA}{\hskip 0.5\mylenA\usebox\myboxB}%
    \fi}
\makeatother

\usepackage[colorinlistoftodos,bordercolor=orange,backgroundcolor=orange!20,linecolor=orange,textsize=scriptsize]{todonotes}

\usepackage{thmtools}

\declaretheorem[within=section]{definition}

\declaretheorem[sibling=definition]{proposition}

\declaretheorem[sibling=definition]{lemma}

\declaretheorem[sibling=definition]{problem}

%% file: sections/intro.tex
\section{Introduction} \label{sec:introduction}
Optimal Transport (OT) is by now an established tool in the machine learning playbook, one that plays a key role when matching probability distributions.
OT has played a prominent role in generative modeling \citep{montavon2016wasserstein, genevay2018,tong2023improving,neklyudov2023action,lipman2022flow,pmlr-v202-neklyudov23a}, adversarial training \citep{sinha2018certifying, wong2019wasserstein}, domain adaptation \citep{courty2017joint}, neuroscience~\citep{janati2020multi}, or single-cell modeling \citep{schiebinger2019optimal,tong2020trajectorynet,bunne2023learning}. However, OT is not without drawbacks, and we target in this work an important pain point that hinders the use of OT in practical tasks: the challenges of choosing a ground cost function to compare two measures, particularly when they are supported on \textit{heterogeneous} and \textit{high-dimensional} Euclidean spaces.

\textbf{Choosing a cost or going quadratic?}
OT is most often used to match probability distributions supported on the same space, in which case the ground cost is often picked to be an (exponentiated) distance on that space. Yet, an increasing number of applications require handling multimodality, i.e. comparing distributions supported on heterogeneous spaces, as in multi-omics singe-cell data matching~\citep{demetci2022scot,klein2023mapping}. 
For such problems, defining a ground cost \textit{across} two spaces (e.g. between vectors of different dimensions), in the absence of other knowledge, is understandably difficult. As a result, practitioners often bypass this issue by resorting to the Gromov-Wasserstein (GW) framework~\citep{memoli2011gromov}. GW instead instantiates a \textit{quadratic} problem in the space of couplings that involves \textit{two} within ground cost functions, one for each space. 
However, solving the GW problem is fraught with computational challenges: it requires minimizing a concave quadratic function on the space of couplings, an NP hard problem. 
Moreover, some of the classic linear OT machinery, e.g. the existence of \citeauthor{monge1781memoire} maps, does not directly translate to quadratic settings, as shown by \citet{vayer2020contribution, dumont2022existence} who discussed their existence for a variety of costs. This is a very active field of research, as shown by recent contributions on the properties of entropic GW \citep{zhang2022gromov,rioux2023entropic}; closed-form formulas for Gaussians~\citep{salmona2022gromov,le2022entropic}, or scalable computational schemes~\citep{pmlr-v162-scetbon22b, nekrashevich2023neural}.

\textbf{High-dimensional issues.} Another issue with OT comes from the significant statistical challenges that arise when dealing with high-dimensional problems \citep{dudley1966weak, weed2019sharp, genevay2019sample}. One way around this is performing offline dimensionality reduction, typically via PCA or VAEs, or more elaborate schemes that carry out a joint estimation of a projection and transport \citep{bonneel2015sliced, niles2022estimation, paty2019subspace, deshpande2019max, le2019tree}. A recent line of work proposes to leverage sparsity in displacements ~\citep{cuturi2023monge, klein2023learning}. These challenges only add up when dealing with heterogeneous spaces and taking the quadratic route advocated in GW. For example, the standard iterative algorithm to solve the GW problem requires solving a linear OT problem \textit{at each iteration} \citep{peyre2016gromov}, and is only able to handle dimensionality as a way to speed up these routines~\citep{pmlr-v162-scetbon22b}

\textbf{Contributions.} In this work, we target both issues in the case where practitioners are faced with the task of transporting points across two modalities, where each lies in a high-dimensional space. 
\begin{itemize}[leftmargin=.3cm,itemsep=.1cm,topsep=0cm,parsep=1pt]
    \item We use a result by~\citet{paty2020regularized}, who showed that minimizing a concave function over couplings can be expressed as the minimization of a cost-regularized OT problem, to show that most common quadratic OT formulations can be unified under that ``cost-regularized OT'' lens.
    \item We propose to leverage this fact to impose explicitly structure on costs using a suitable convex regularization. Focusing on dot-product costs parameterized by a \textit{linear transform} across spaces, we propose a simple alternated minimization algorithm, \textsc{Prox-ROT}, to solve cost-regularized OT when the regularizer is evaluated on such transforms. \textsc{Prox-ROT} is particularly efficient when the regularization function has a closed-form proximity operator.
    \item We use this framework to introduce sparsifying norms on such transforms and show this is equivalent to variants of GW which only depend on an \textit{adaptively selected} subset of dimensions: Using \textsc{Prox-ROT} with such norms consists in iteratively performing a feature selection step followed by an OT plan computation step until convergence. This makes cost-regularized OT particularly well-suited to high-dimensional problems, where linear OT and GW usually struggle.
    \item Building on \cite{dumont2022existence}, we show that there exist Monge maps for the linear cost-regularized OT problem. Extending the framework of entropic Monge maps \citep{pooladian2021entropic}, we derive entropic Monge maps \textit{across spaces}, for linear cost-regularized OT, and demonstrate that they converge to the ground truth Monge maps under suitable assumptions. 
    \item We apply our methods to toy and real-world data. We use sparsifying and low-rank regularizations to solve high-dimensional mutliomics single-cell data integration tasks while improving over the best-known OT baselines, and we take advantage of the cost-regularized OT objective structure to use SGD to solve a large-scale spatial transcriptomics problem.
\end{itemize}
The closest setting to ours appeared in \cite{alvarez2019towards}, who proposed to use a constraint on the spectrum of the linear transforms and showed the equivalence between the Gromov-Wasserstein problem and solving a linear OT problem with a dot product cost using a Frobenius norm constraint on the linear operator. The \textsc{Prox-ROT} algorithm in Section \ref{sec:struct_ot} using the nuclear or the rank regularization results (up to reparametrization) in a similar algorithm to the one proposed in \cite{alvarez2019towards}. We take the alternative \emph{regularization} route and consider general regularizations. Notably, we pay particular attention to sparsifying norms acting directly on the entries of the linear operator, and draw links with GW-like problems on a subset of the dimensions.

\textbf{Notations.}
In what follows, we consider $\br{\cX, d_\cX}$ and $\br{\cY, d_\cY}$, two metric spaces, and $\br{\alpha, \beta} \in \cP(\cX) \times \cP(\cY)$, two probability measures with compact supports. The space of couplings $\Pi(\alpha, \beta)$ is defined as the probability distributions on $\br{\cX \times \cY}$ whose marginals are $\alpha$ and $\beta$ \citep{santambrogio2015optimal}:
\begin{align}
    \Pi\br{\alpha, \beta} = \left\{ \gamma \in \cP\br{\cX \times \cY}: \br{\pi_x}_\sharp \gamma = \alpha, \br{\pi_y}_\sharp \gamma = \beta  \right\},
\end{align}
where $\pi_x$ and $\pi_y$ are the projections of $\cX \times \cY$ onto $\cX$ and $\cY$ respectively. We denote the linear OT cost between $\alpha$ and $\beta$ with cost $c \in \cC(\cX \times \cY)$ as
\begin{align}
    \cW_c\br{\alpha, \beta} \triangleq \min_{\pi \in \Pi(\alpha, \beta)} \int_{\cX \times \cY} c(x, y) \diff \pi(x, y),
\end{align}
which is a \textit{linear} problem in $\pi$. When $c = d_\cX^p$ and $\cX = \cY$, $\cW_c^{1/p}$ defines a distance between probability measures for all $p \geq 1$ \citep{villani2003topics}.

%% file: sections/background.tex
\section{OT Across Spaces as a Concave Minimization Problem}\label{sec:background}
While the classical OT objective is a \textit{linear} function of the coupling $\pi$, existing approaches to comparing probability distributions across spaces all require minimizing a \textit{concave} function over $\Pi(\alpha, \beta)$.

\begin{problem}\label{pb:conc_min}
    Let $\cQ$ be a concave function over $\Pi(\alpha, \beta)$. We define the concave minimization over couplings problem as
    \begin{align}\label{eq:QOT}
        \min_{\pi \in \Pi(\alpha, \beta)} \cQ(\pi). \tag{$\cQ \textsc{OT}$}
    \end{align}
\end{problem}
Remarkably, concave minimization problems over couplings are closely linked to linear OT via convex duality.
\begin{proposition}[Remark 2 in \cite{paty2020regularized}]
    Let $\cQ$ be a proper usc concave function over $\Pi(\alpha, \beta)$. Writing $\cQ^*$ for its convex conjugate, one has:
    \begin{align}
        \inf_{\pi \in \Pi(\alpha, \beta)} \cQ(\pi) & = \inf_{\substack{\pi \in \Pi(\alpha, \beta) \\ c \in \cC(\cX \times \cY)}} \int c(x,y) \diff \pi + (- \cQ)^*(-c)      \\
        & = \inf_{c \in \cC(\cX \times \cY)} \cW_c\br{\alpha, \beta} + (- \cQ)^*(-c).    
    \end{align}
\end{proposition}
In other words, concave minimization problems over couplings can be reformulated as the minimization over cost functions of a linear OT problem plus a cost regularization. Consequently, we can define the following cost-regularized reformulation of Problem \ref{pb:conc_min}.
\begin{problem}\label{pb:cost_reg_ot}
    Let $\cR$ be a convex function over $\cC\br{\cX \times \cY}$. We define the cost-regularized optimal transport problem as
    \begin{align}\label{eq:ROT}
        \cR\textsc{OT}(\alpha, \beta) \triangleq \inf_{c \in \cC(\cX \times \cY)} \cW_c\br{\alpha, \beta} + \cR(c). \tag{$\cR \textsc{OT}$}
    \end{align}
\end{problem}
Note that, trivially, \ref{eq:ROT} $=$ \ref{eq:QOT} when $\cR(c) = \br{-\cQ}^*(-c)$. 
We show next how well-known OT formulations used to compare probability distributions across spaces can be reduced to either \ref{eq:QOT} or \ref{eq:ROT}.
We say that two minimization problems are equivalent, or that one is an instance of the other, when they have the same minimizers in $\Pi(\alpha, \beta)$.

\subsection{Examples}
\textbf{Sturm's distance.} \cite{sturm2006geometry} defined a distance between metric measured spaces (MMS) that quantifies how ``isometric`` they are. Given two MMS $(\cX, d_X, \alpha)$ and $(\cY, d_Y, \beta)$, \citeauthor{sturm2006geometry} defined the squared distance:
\begin{align}
    \min_{\pi \in \Pi(\alpha, \beta)} \min_{\substack{(Z, d_Z) \in \cC \\ \phi : \cX \to \cZ \\ \psi : \cY \to \cZ}} \int_{\cX \times \cY} d_Z^2(\phi(x), \psi(y))\,\mathrm{d}\pi(x, y),
\end{align}
where $\phi : \cX \to \cZ$ and $\psi : \cY \to \cZ$ are constrained to be isometries and $\br{\cZ, d_{\cZ}}$ is a metric space. As such, Sturm's distance is an instance of \ref{eq:ROT} with
\begin{gather}
    \cR(c) = \left\{
    \begin{array}{ll}
        0 & \mbox{if } \exists (\cZ, \phi, \psi): c(x,y) = d_Z^2(\phi(x), \psi(y)) \\
        \infty & \mbox{otherwise.}
    \end{array}
\right.
\end{gather}
\textbf{Wasserstein Procrustes.} 
A more tractable variant of Sturm's distance when $\cX, \cY \subset \R^{d}$ is the so-called Wasserstein Procrustes problem \citep{zhang2017earth, grave2019unsupervised}. It is defined as
\begin{align}\label{eq:wasserstein_procrustes}
    \min_{\pi \in \Pi(\alpha, \beta)}\min_{\m{C} \in \cO_{d}} \int_{\cX \times \cY} \norm{\m{C}x - y}^2\diff \pi(x,y).
\end{align}
where $\cO_d$ is the orthogonal manifold. Wasserstein Procrustes defines a distance between $\cX$ and $\cY$ up to a rotation. It is an instance of \ref{eq:ROT} with
\begin{gather}
    \cR(c) = \left\{
    \begin{array}{ll}
        0 & \mbox{if } \exists \m{C} \in \cO_d: c(x,y) = \norm{\m{C}x - y}^2. \\
        \infty & \mbox{otherwise.}
    \end{array}
\right.
\end{gather}
Remarkably, \eqref{eq:wasserstein_procrustes} is a weighted orthogonal Procrustes problem, and can be cast as an instance of \ref{eq:QOT}.
\begin{proposition}\label{prop:Q_wasserstein_procrustes}
    Let $\cX, \cY \subset \cR^{d}$. Then the Wasserstein Procrustes distance is an instance of the concave minimization problem \eqref{eq:QOT} with
    \begin{align}
        \cQ(\pi) = \scaleobj{1.2}{\int} \norm{\mU(\pi)\mV(\pi)^\top x - y }^2 \diff \pi,
    \end{align}
    where $\mU(\pi)\mSigma(\pi) \mV(\pi)^\top$ is an SVD of ${\int_{\cX \times \cY} y x^\top \diff \pi}$.
\end{proposition}
\textbf{Arbitrary Transformations.} While Sturm considers isometries and the Wasserstein Procrustes approach uses rotations, any type of a transformation of $\alpha$ and/or $\beta$ could be used. For example, \cite{cohen1999earth} studied a large range of transformations, including linear ones and translations. Finally, learning generative models with OT \citep{genevay2018,salimans2018improving} with learned features to define costs is another instance of \ref{eq:ROT}.

\textbf{The Gromov-Wasserstein distance.} Given two costs $c_\cX$ and $c_\cY$, the Gromov-Wasserstein (GW) distance \citep{memoli2011gromov} is defined as
\begin{align}
    \min_{\pi \in \Pi(\alpha, \beta)} \int_{(\cX \times \cY)^2} \br{c_\cX(x,x^\prime) - c_\cY(y,y^\prime)}^2 \dif\pi \diff \pi. \label{eq:GW}
\end{align}
\cite{memoli2011gromov} showed that when $c_\cX$ and $c_\cY$ are distances, \eqref{eq:GW} defines a distance between metric measured spaces up to isometry. This problem is not concave in general, but when $c_\cX$ and $c_\cY$ define a conditionally negarive kernel, the function $\pi \mapsto \ip{Q(\pi)}{\pi}$, where
\begin{align}
    Q \br{(x,y), (x^\prime, y^\prime)} := \br{c_\cX(x,x^\prime) - c_\cY(y,y^\prime)}^2,
\end{align}
is a concave quadratic function of $\pi$ \citep[Proposition 1]{dumont2022existence}. In this case, \eqref{eq:GW} is an instance of \eqref{eq:QOT} with $\cQ(\pi) = \ip{Q(\pi)}{\pi}$.

Two common cases that fit this setting are the inner product and the squared Euclidean distance. In this case, it has been recently shown \citep{vayer2020contribution} that the concave minimization problem has a simple cost-regularized equivalent.
\begin{proposition}\label{prop:GW_equivalence_ROT}
    Let $\cX \subset \R^{d_x}$ and $\cY \subset \cR^{d_y}$. For all $\mM  \in \R^{d_y \times d_x}$, define $c_{\mM }(x,y) = - \ip{\mM x}{y}$.
    
    \textbullet\ Let $c_\cX = \ip{\cdot}{\cdot}_\cX$ and $c_\cY = \ip{\cdot}{\cdot}_\cY$. The problems 
    \begin{align}
        &\min_{\pi \in \Pi(\alpha, \beta)} \int \br{\ip{x}{x^\prime} - \ip{y}{y^\prime}}^2 \diff \pi \diff \pi \label{eq:GW-IP} \tag{GW-IP} \\
        \text{and }&\min_{\substack{\pi \in \Pi(\alpha, \beta) \\ \mM  \in \R^{d_y \times d_x}}} \int - \ip{Mx}{y} \diff \pi + \frac{1}{2}\norm{M}_F^2 \label{eq:GW-IP-ROT} \tag{GW-IP-$\cR \textsc{OT}$}
    \end{align}
    are equivalent: \ref{eq:GW-IP} is an instance of \ref{eq:ROT} with
    \begin{gather}
    \cR(c) = \left\{
    \begin{array}{ll}
        \frac{1}{2}\norm{\mM }^2_F & \mbox{if } \exists \mM  \in \R^{d_y \times d_x} : c = c_\mM  \\
        \infty & \mbox{otherwise.}
    \end{array}
\right.
\end{gather}
    \textbullet\ Let $c_\cX = \sqn{\cdot - \cdot}_\cX$ and $c_\cY = \sqn{\cdot - \cdot}_\cY$. The problems 
    \begin{align}
        &\min_{\pi \in \Pi(\alpha, \beta)} \int (\sqn{x-x^\prime} - \sqn{y-y^\prime})^2 \diff \pi \diff \pi \label{eq:GW-Sq} \\
        \text{and }&\min_{\substack{\pi \in \Pi(\alpha, \beta) \\ \mM  \in \R^{d_y \times d_x}}} \int - \ip{Mx}{y} - \sqn{x}\sqn{y} \diff \pi + \frac{1}{2}\norm{M}_F^2
    \end{align}
    are equivalent. Hence, \eqref{eq:GW-Sq} is an instance of \ref{eq:ROT} with
        \begin{gather}
        \cR(c) = \left\{
        \begin{array}{ll}
            \frac{1}{2}\norm{\mM }^2_F & \mbox{if } \exists \mM : c = c_{\mM } - \sqn{x}\sqn{y} \\
            \infty & \mbox{otherwise.}
        \end{array}
    \right.
\end{gather}
\end{proposition}

\subsection{Leveraging cost-regularized OT}
In this work, we aim to incorporate additional structure to OT across spaces by leveraging the \ref{eq:ROT} formulation. Inspired by GW, we will constrain our costs to be linear and consider new regularization functions. We will see that depending on the regularization function, we can draw interesting links to GW-like problems by writing the \ref{eq:ROT} problem in its equivalent \ref{eq:QOT} form.

\textbf{Entropic regularization.} A preferred way to compute OT and GW distances in practice is adding an entropic regularization to the linear OT term \citep{cuturi2013sinkhorn, peyre2016gromov}. Given an $\varepsilon > 0$ we can follow the same procedure and add a regularization to \ref{eq:ROT}. The entropy-regularized \ref{eq:ROT} problem writes
\begin{gather}\label{eq:ROT_ent}
    \min_{c \in \cC(\cX \times \cY)} \cW_{c}^\varepsilon(\alpha, \beta) + \cR(c) \\
    \text{where} \; \, \cW_{c}^\varepsilon(\alpha, \beta) \triangleq \min_{\pi \in \Pi(\alpha, \beta)} \int c \diff \pi + \varepsilon \textsc{KL}(\pi || \alpha \otimes \beta),
\end{gather}
with $\textsc{KL}(\pi || \alpha \otimes \beta) = \int \log\br{\frac{\diff \pi}{\diff \alpha \otimes \beta}} \diff \pi$. In this case, the entropic version of \ref{eq:QOT} is simply
\begin{align}\label{eq:QOT_ent}
    \min_{\pi \in \Pi(\alpha, \beta)} \cQ(\pi) + \varepsilon \textsc{KL}(\pi || \alpha \otimes \beta).
\end{align}

\textbf{Fused $\cR\textrm{OT}$.} In some cases, we may even have an inter-space cost that can be used alongside the learned cost. This is the case in the context of the fused Gromov-Wasserstein cost \citep{vayer2020fused}. Such a cost $\tilde{c} : \cC(\cX \times \cY) \to \R$, which isn't learned, can be naturally integrated as (with $\eta \geq 0$)
\begin{gather}\label{eq:ROT_ent_fused}
    \cR OT_\varepsilon^{\tilde{c}}(\alpha, \beta) = \min_{c \in \cC(\cX \times \cY)} \cW_{c + \eta\tilde{c}}^\varepsilon(\alpha, \beta) + \cR(c).
\end{gather}

%% file: sections/method.tex
\section{Structured OT across spaces}\label{sec:struct_ot}
\subsection{Linear cost-regularized OT}

Let $\cX \subset \R^{d_x}$ and $\cY \subset \R^{d_y}$. From this point on, we consider linear costs parameterized by a matrix $\mM  \in \R^{d_y \times d_x}$: ${c_\mM (x, y) = - \ip{\mM x}{y}}$. In this case we can write \eqref{eq:ROT_ent} as
\begin{gather}\label{eq:linear_rot}
    \min_{\substack{\pi \in \Pi(\alpha, \beta) \\ \mM \in \R^{d_y \times d_x}}} \int - \ip{\mM x}{y} \diff \pi + \cR(\mM) + \varepsilon \textsc{KL}(\pi || \alpha \otimes \beta), \tag{L-$\cR \textsc{OT}_\varepsilon$}
\end{gather}
where we directly define $\cR$ on $\R^{d_y \times d_x}$ for brevity.
\begin{proposition}\label{prop:Q_linear_ROT}
The linear $\cR \textsc{OT}$ problem \eqref{eq:linear_rot} is equivalent to the $\cQ \textsc{OT}$ problem \eqref{eq:QOT_ent} with
\begin{align}
    \cQ(\pi) = - \cR^*\br{\int_{\cX \times \cY} yx^\top\diff \pi(x,y)},
\end{align}    
where $\cR^*$ is the convex conjugate of $\cR$.
\end{proposition}
Since for low values of $\varepsilon$, \eqref{eq:linear_rot} is a concave minimization problem, solving it exactly is generally out of reach. However, we can approximate it using a simple alternated minimization procedure on the objective. Initializing at a $\pi_0 \in \Pi(\alpha, \beta)$, let
\begin{align}\label{eq:alt_min}
\begin{split}
    \mM_{k+1} & = \argmin_{\mM \in \R^{d_y \times d_x}}  \int - \ip{\mM x}{y} \,\mathrm{d}\pi_k + \cR(\mM) \\
    \pi_{k+1} & = \argmin_{\pi \in \Pi(\alpha, \beta)} - \ip{\mM_{k+1} x}{y} \,\mathrm{d}\pi + \varepsilon \textsc{KL}(\pi || \alpha \otimes \beta)
\end{split}
\end{align}

\begin{proposition}\label{prop:convergence_alt_min}
    Let $\cX = \left\{x_i\right\}_{i=1}^m \subset \R^{d_x}$, $\cY = \left\{y_i\right\}_{i=1}^n \subset \R^{d_y}$, $\alpha = \sum_{i=1}^m \balpha_i \delta_{x_i}$ and $\beta = \sum_{j=1}^n \bbeta_j \delta_{y_j}$. Let $\varepsilon > 0$ and $\cR$ be strongly convex. Then, any limit point of $\br{M_{k}, \pi_k}$ defined in \eqref{eq:alt_min} is a stationary point of the objective in \eqref{eq:linear_rot}. 
\end{proposition}

In practice, the $\pi$-step can be carried out efficiently using Sinkhorn's algorithm \citep{cuturi2013sinkhorn} with the cost $c_{\mM_k}$. The $\mM$-step is in general trickier since it might involve a difficult optimization problem depending on the regularization function $\cR$.

\subsection{Proximal $\cR\textsc{OT}$}\label{subsec:prox-rot}

We consider regularizations of the form ${\cR(\mM) = \frac{1}{2}\norm{\mM}^2_F + \lambda g(\mM)}$. We aim to solve

\begin{align}
    \min_{\substack{\pi \in \Pi(\alpha, \beta) \\ \mM \in \R^{d_y \times d_x}}} - & \int \ip{\mM x}{y}\,\mathrm{d}\pi + \frac{\norm{\mM}_F^2}{2} + \lambda g(\mM) \label{eq:reg-OT_g} \\
    & + \varepsilon \textsc{KL}(\pi || \alpha \otimes \beta),
\end{align}
where $g : \R^{d_y \times d_x} \to \R$. Note that when $\lambda = 0$, per Prop. \ref{prop:GW_equivalence_ROT}, \eqref{eq:reg-OT_g} is equivalent to \eqref{eq:GW-IP}. The minimization with respect to $\mM$ in \eqref{eq:reg-OT_g} is a well-studied problem in the optimization and machine learning literature \citep{bauschke2011convex}. 
\begin{lemma}\label{lem:prox_M}
    For any $\pi$, the solution of \eqref{eq:reg-OT_g} in $\mM$ is
    \begin{align}
        & \mM(\pi) \triangleq \prox_{\lambda g}\br{\int_{\cX \times \cY} yx^\top \,\mathrm{d}\pi},
    \end{align}
    where $\prox_{h}(x) \triangleq \argmin_{z \in \R^d} \frac{1}{2}\sqn{x - z} + h(z)$.
\end{lemma}
The alternated minimization scheme \eqref{eq:alt_min} becomes:
\begin{mdframed}
\begin{align}
    \mM_{k+1} = \prox_{\lambda g}& \br{\int_{\cX \times \cY} yx^\top \diff \pi_k} \label{alg:Prox-ROT} \tag{Prox-$\cR\textsc{OT}$} \\
    \pi_{k+1} = \argmin_{\pi \in \Pi(\alpha, \beta)} & - \int_{\cX \times \cY} \ip{\mM_{k+1}x}{y}\diff \pi_k \\
    & + \varepsilon \textsc{KL}(\pi_k || \alpha \otimes \beta)
\end{align}
\end{mdframed}
As a result, the $\mM$-step in \ref{alg:Prox-ROT} can be efficiently implemented when the proximal operator of $g$ can be computed in closed form.

\textbf{Statistical benefits of sparsity.} Since OT famously suffers from the curse of dimensionality -- the sample complexity for estimating OT maps is $\cO(n^{-\sfrac{1}{d}})$ \citep{niles2022estimation, pooladian2021entropic} -- the main factor we consider in choosing the regularization, beyond computational efficiency, is using inductive biases to reduce \textit{effective dimension} of the OT problem.

\textbf{$\ell_1$ and $\ell_{1,2}$ regularizations.} Among the choices of $g$ where the $\mM$-step can be implemented efficiently are the $\ell_1$ and and $\ell_{1,2}$ regularizations, where $\norm{\mM}_1 = \sum_{i=1}^{d_x} \sum_{j=1}^{d_y} \abs{\mM_{ij}} \; \text{and} \; \norm{\mM}_{1, 2} = \sum_{i=1}^{d_x} \norm{\mM_{:i}}_2$. Their proximal operators are given by \citep{chierchia2023proximity}
\begin{align}
    \prox_{\lambda \norm{\cdot}_1}\br{\mM}_{ij} & = \sgn(\mM_{ij})\br{\abs{\mM_{ij}} - \lambda}_+\\
    \prox_{\lambda \norm{\cdot}_{1,2}}\br{\mM}_{:i} & = \br{1 - \frac{\lambda}{\max\left\{\norm{\mM_{:i}}\, , \lambda \right\}}}\mM_{:i},
\end{align}
where $(a)_+ = \max(a, 0)$ and $\mA_{:i}$ denotes the $i$th column of $\mA$. Intuitively, the reason behind using $\ell_1$ regularization is to select the features in the spaces (or point clouds) $\cX$ and $\cY$ that are the most helpful in aligning those spaces. Similarly, using $\ell_{1,2}$ regularization, we can discard features in one of the spaces alone. 

This is best seen by examining the iterations of \ref{alg:Prox-ROT} in each of these cases: For $\ell_1$, the prox operator in the $\mM$ iteration results in a matrix $\mM_{k+1}$ where some entries $\mM_{k+1}(i,j)$ can be set to $0$, and as a result the update of the OT plan $\pi_{k+1}$ is oblivious to the corresponding features $x_i$ and $y_j$; similarly, for $\ell_{1, 2}$ iterations, the prox sets some columns of $\mM_{:i}$ to $0$, and as a result the update of the OT plan doesn't use the corresponding dimensions $x_i$. In both cases, \ref{alg:Prox-ROT} alternates between a feature selection step and a linear OT step.

Rewriting \eqref{eq:reg-OT_g} in its \ref{eq:QOT} form, we show that these regularizations result in problems that are closely related to the inner-product GW problem \eqref{eq:GW-IP}.

\begin{proposition}\label{prop:l1_l12}
    Consider Problem \eqref{eq:reg-OT_g} with $\varepsilon=0$.

    \textbullet\ \textbf{$\ell_{1, 2}$-regularization:} with $g(M) = \norm{M}_{1,2}$, \eqref{eq:reg-OT_g} with ${\varepsilon=0}$ is equivalent to \ref{eq:QOT} with
        \begin{align}\label{eq:GW_subset_X}
        & \cQ(\pi) = \int_{\br{\cX \times \cY}^2}  \br{\ip{x_{I_{\pi}}}{x_{I_\pi}^\prime} -  \ip{y}{y^\prime}}^2 \diff \pi \diff \pi - \frac{\lambda^2}{2} |I_\pi| \\ 
        & - \int_\cX \ip{x_{I_\pi}}{x_{I_\pi}^\prime}^2 \diff \alpha + \lambda \sum_{i \in I_\pi} \norm{\int_{\cX \times \cY} x_i y\, \,\mathrm{d}\pi(x, y)},
        \end{align}
        where $I_\pi = \left\{i \,: \; \norm{\int_{\cX \times \cY}x_i y \diff \pi} > \lambda \right\}$. Moreover, with $\br{\pi^\star, M^\star}$ a solution to \eqref{eq:reg-OT_g}, \eqref{eq:reg-OT_g} is equivalent to
    \begin{align}
        \min_{\pi \in \Pi(\alpha, \beta)} \int_{\br{\cX \times \cY}^2} \br{ \ip{x_{I_{\pi^\star}}}{x_{I_{\pi^\star}}^\prime} - \ip{y}{y^\prime}}^2 \diff \pi^\star \diff \pi.
    \end{align}
    \textbullet\ \textbf{$\ell_1$-regularization:} With $g(M) = \norm{M}_1$, \eqref{eq:reg-OT_g} with ${\varepsilon=0}$ is equivalent to \ref{eq:QOT} with
    \begin{align} \label{eq:GW_subset_XY}
        & \cQ(\pi) =  \int_{\br{\cX \times \cY}^2} \sum_{(i, j) \in I_{\pi}} \br{x_i x_i^\prime - y_j y_j^\prime}^2 \diff \pi \diff \pi - \frac{\lambda^2}{2}|I_\pi| \\
        & - \sum_{(i,j)\in I_\pi}\br{\int_\cX x_i^2 \diff \alpha + \int_\cY y_j^2 \diff \beta  - \lambda \abs{\int_{\cX \times \cY} y_j x_i \,\mathrm{d}\pi(x, y)}},
    \end{align}
    where $I_{\pi} = \left\{(i, j) \,: \; \int_{\cX \times \cY} x_iy_j \,\mathrm{d}\pi > \lambda \right\}$. Moreover, with $\br{\pi^\star, M^\star}$ a solution to \eqref{eq:reg-OT_g}, \eqref{eq:reg-OT_g} is equivalent to
    \begin{align}
        \min_{\pi \in \Pi(\alpha, \beta)} \int_{\br{\cX \times \cY}^2} \sum_{(i, j) \in I_{\pi^\star}} \br{x_i x_i^\prime - y_j y_j^\prime}^2 \diff \pi^\star \diff \pi.
    \end{align}
\end{proposition}

\textbf{Nuclear and Rank regularizations.} Two other possible choices of regularizations are the nuclear norm and rank regularizations.  Let $\mM = \mU^\top\mSigma \mV$ be an SVD of $\mM$, where $\mSigma$ contains the vector of singular vectors in decreasing order $\sigma = \br{\sigma_i}_{i \in [d_x]}$ of $\mM$. The nuclear norm of $\mM$ is defined as $\norm{\mM}_* = \norm{\sigma}_1$. The proximal operator of $\norm{\cdot}_*$ and $\rk(\cdot)$ are respectively
\begin{align}
    \prox_{\lambda \norm{\cdot}_*}(\mM) & = \mU^\top \prox_{\lambda \norm{\cdot}_1}\br{\sigma} \mV \\
    \prox_{\lambda \rk}(\mM) & = \mU^\top \prox_{\lambda \norm{\cdot}_0}\br{\sigma} \mV,
\end{align}
where $\prox_{\lambda \norm{\cdot}_0}(\sigma)_i = \sigma_i \mathbbm{1}_{\left\{\sigma_i^2 > 2\lambda\right\}}$. Consider a solution $(\mM^\star, \pi^\star)$ of \eqref{eq:reg-OT_g} with the nuclear norm. Then $\mM^\star = \prox_{\lambda \norm{\cdot}_*}\br{\int yx^\top \diff \pi^\star}$. Let $\lambda$ be large enough so that $\rk(\mM^\star) = r > 0$. Denote by ${\tilde{\mU} = \mU_{:, :r}}$  (resp. ${\tilde{\mV} = \tilde{\sigma} \odot \mV_{:, :r}}$) the restriction of $\mU$ (resp. $\mV$) to its first $r$ lines (resp. columns multiplied elementwise with $\tilde{\sigma}$, the restriction of $\sigma$ to its first $r$ elements). Then, we can write the cost of the linearized problem as $c_{\mM^\star}(x, y) = - \ip{\tilde{\mV}x}{\tilde{\mU}y}$. $\tilde{\mV}x, \tilde{\mU}y \in \R^{r}$, thus the nuclear norm allows to similarly reduce the effective dimension of the problem from potential large $d_x, d_y$ to a small $r$ determined by the magnitude of the regularization $\lambda$.

Note that since computing the SVD at each iteration of \ref{alg:Prox-ROT} can be costly when $d_x$ and $d_y$ are large, we suggest in extremely high-dimensional problems to directly parametrize $\mM$ as a low-rank matrix $\mM = \mM_2^\top \mM_1$, where $r \ll \min(d_x, d_y)$, and use the alternated minimization scheme \eqref{eq:alt_min}. We write the corresponding iterations in the Appendix. Despite making the problem non-convex in $\mM$, we see in our applications that the resulting method works well.

%% file: sections/maps.tex
\section{Entropic Monge maps for $\cR\textsc{OT}$}\label{sec:maps}
Optimal transport \citep{monge1781memoire} seeks a map ${T : \cX \to \cY}$ that minimizes the average displacement cost $c(x, T(x))$ between two measures $\alpha$ and $\beta$, i.e. finding an OT coupling of the form $\pi^\star = \br{\textrm{id}, T}_\sharp \alpha$. Although such a (so-called) Monge map does not exist for all costs, \cite{brenier1991polar} showed that in the case where $c(x, y) = \tfrac{1}{2}\sqn{x - y}$, if $\alpha$ has a density, the optimal map exists, is unique, and can be written as the gradient of a convex function.
A convenient approach to approximate that Monge map can be found in entropic regularization. \citet{pooladian2021entropic} showed that one can build a map $T_\varepsilon$ using entropy regularized transport (with regularization strength $\varepsilon>0$) such that $T_\varepsilon \xrightarrow[]{\varepsilon \to 0} T$.
We show that there always exists a Monge map for linear $\cR \textsc{OT}$, and propose a formulation for an entropic map for this problem.

\subsection{Monge maps for linear $\cR \textsc{OT}$}
\cite{dumont2022existence} recently showed that there exists a Monge map for the \ref{eq:GW-IP} problem, which, as seen in Prop. \ref{prop:GW_equivalence_ROT} is a special case of linear $\cR \textsc{OT}$. Here, we extend their reasoning to show the existence of Monge maps for the general linear $\cR \textsc{OT}$ problem.
\begin{proposition}\label{prop:monge_map_linear_costs}
    Let $\alpha \in \cP(\R^{d_x})$ and $\beta \in \cP(\R^{d_y})$ with compact supports and $d_x \geq d_y$. Assume that $\alpha \ll \cL^{d_x}$, the Lebesgue measure. Then there exists a map $T : \cX \to \cY$ such that $\pi^\star = \br{\textrm{id}, T}_\sharp \alpha$ and $\br{\pi^\star, \mM(\pi^\star)}$ is optimal for \eqref{eq:linear_rot} with $\varepsilon=0$.
\end{proposition}
\begin{proof}
Let $(\mM^\star, \pi^\star)$ be optimal for \eqref{eq:linear_rot} with $\varepsilon=0$. Then $\pi^\star$ is a solution to
\begin{align}
    \min_{\pi \in \Pi(\alpha, \beta)} \int_{\cX \times \cY} -\ip{\mM^\star x}{y} \diff \pi(x,y),
\end{align}
and \cite[Theorem 4]{dumont2022existence} showed that for all $\mM \in \R^{d_y \times d_x}$, under the assumptions of Prop. \ref{prop:monge_map_linear_costs}, there exists a Monge map for the cost ${c_{\mM}(x,y) = - \ip{\mM x}{y}}$ between $\alpha$ and $\beta$.
\end{proof}
Now that the existence of a Monge map has been established, a natural question is whether, and how, we can approximate it using entropy-regularization. In the following sections, we define an entropic map for \ref{eq:linear_rot} and show its convergence to a Monge map for this problem under suitable assumptions. Then, we discuss its convergence for the sparse and low-rank regularizations we considered in Section \ref{subsec:prox-rot}.

\subsection{Entropic Monge maps for linear $\cR\textsc{OT}$}
\begin{definition}[Entropic map for \ref{eq:linear_rot}]\label{def:ent_map}
Let $\br{\pi_\varepsilon^\star, M_\varepsilon^\star}$ be a solution of \ref{eq:linear_rot}. Let $\varepsilon^\prime > 0$. Define
\begin{align}\label{eq:ent-map}
    T_{\varepsilon, \varepsilon^\prime}(x) = \frac{\int y \exp{\br{\br{g_{\varepsilon, \varepsilon^\prime}(y) + \ip{M_\varepsilon^\star x}{y}}/\varepsilon^\prime}}d\beta(y)}{\int \exp{\br{\br{g_{\varepsilon, \varepsilon^\star}(y) + \ip{M_\varepsilon^\star x}{y}}/\varepsilon^\prime}}d\beta(y)}.
\end{align}
Here, $(f_{\varepsilon, \varepsilon^\prime}, g_{\varepsilon, \varepsilon^\prime})$ are Sinkhorn potentials for the inner product cost between ${M_\varepsilon^\star}_\sharp \alpha$ and $\beta$ with an $\varepsilon^\prime$ entropic regularization: They are a solution of the problem
\begin{gather}
    \max_{(f, g) \in \cC(X) \times \cC(Y)} \int_{X}f(x) d{\mM_\varepsilon^\star}_\sharp \alpha(x) + \int_{Y}g(y) d\beta(y) \\
   - \varepsilon^{\prime} \int_{X \times Y} \exp\br{\frac{f(x) + g(y) + \ip{x}{y}}{\varepsilon^\prime}} \diff {\mM_\varepsilon^\star}_\sharp \alpha \otimes \beta(x,y).
\end{gather}
\end{definition}

\subsubsection{Computing the entropic map in practice} 
An important point in the theoretical definition of the entropic map \eqref{eq:ent-map} is to use two different regularization parameters $\varepsilon, \varepsilon^\prime > 0$. In practice (as we do in our applications), we can simplify the definition and use a single $\varepsilon$. Given samples $\left\{x_i\right\}_{i=1}^n \sim \alpha$ and $\left\{y_i\right\}_{i=1}^n \sim \beta$:\\
\textbullet\ Find a primal solution $(\mpi^\star, \mM_\varepsilon^\star)$ and a dual solution $\br{\mathbf{f}, \mathbf{g}}$ of the discrete \ref{eq:linear_rot} problem with $\hat{\alpha} = \sum_{i=1}^n \balpha_i \delta_{x_i}$ and $\hat{\beta} = \sum_{j=1}^m \bbeta_j \delta_{y_j}$ using \eqref{eq:alt_min}. \\
\textbullet\ Define the entropic map as
\begin{align}\label{eq:ent_map_practical}
    T_{\varepsilon}(x) = \frac{\sum_{j=1}^n y_j \exp{\br{\br{\mathbf{g}_j + \ip{\mM_\varepsilon^\star x}{y_j}}/\varepsilon}}}{\sum_{k=1}^n \exp{\br{\br{\mathbf{g}_k + \ip{\mM_\varepsilon^\star x}{y_k}}/\varepsilon}}}.
\end{align}

\begin{figure}[t]
\captionsetup{font=footnotesize}
\centering
\includegraphics[width=\linewidth]{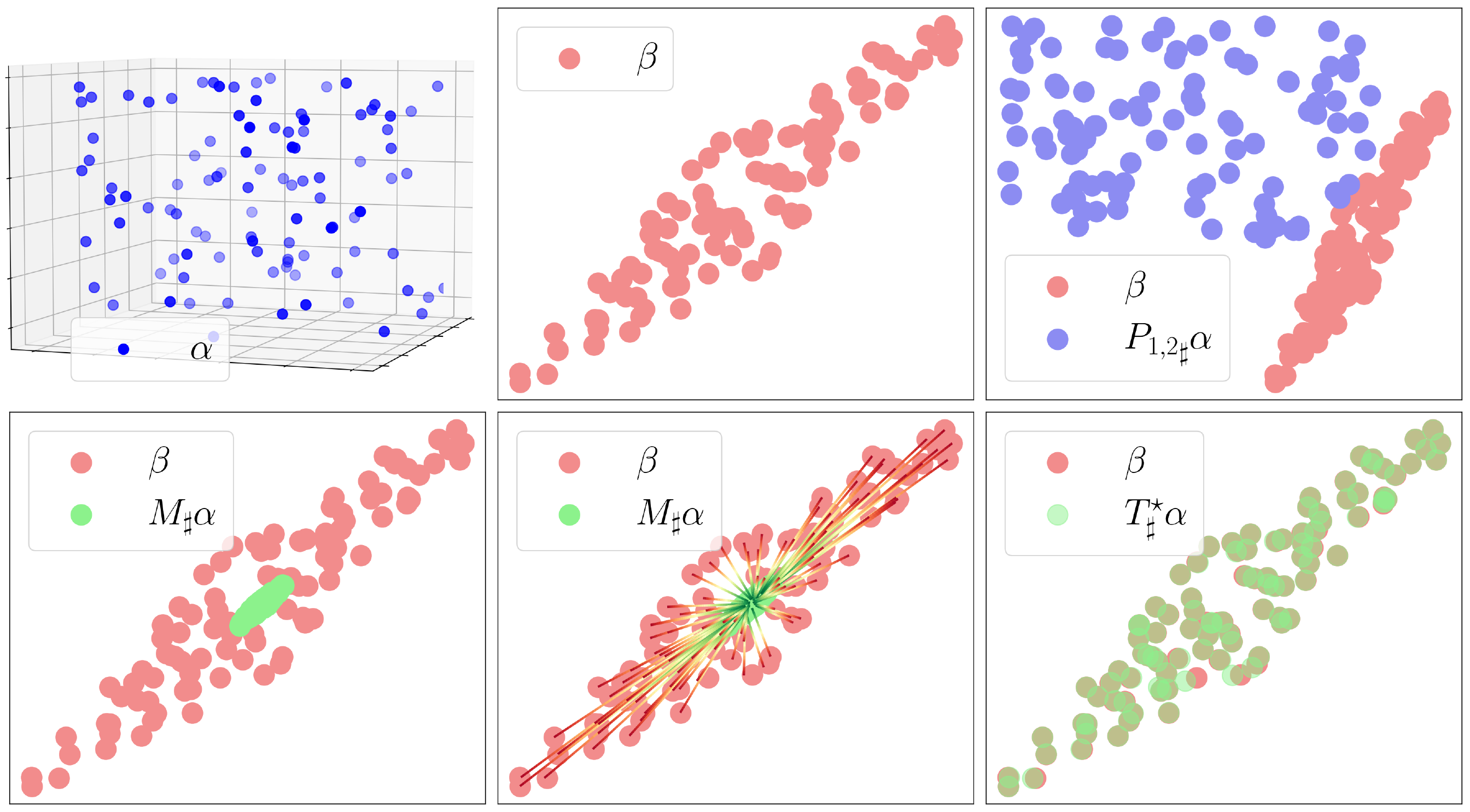}
\caption{Illustration of the entropic map \eqref{eq:ent_map_practical} for \ref{eq:GW-IP} from a 3D to a 2D point cloud with $\varepsilon \approx 0$.}
\label{fig:entropic_map}
\end{figure}

\textbf{Illustration.} In Figure \ref{fig:entropic_map}, we showcase the action of the entropic map \eqref{eq:ent_map_practical} when $\varepsilon \approx 0$. We consider two point clouds $\alpha$ and $\beta$ on $\R^{3}$ and $\R^{2}$. The entropic map implicitly acts in two steps: it first transforms the 3D point cloud $\alpha$ into a 2D point cloud $\mM_\sharp \alpha$ which is aligned with $\beta$. It then maps the points from $\mM_\sharp \alpha$ to $\beta$ using the inner-product entropic map.

\subsubsection{Convergence theory}
\textbf{The case where $\mM^\star$ is full rank.} Assume that $d_x \geq d_y$ and consider the case where $rk(\mM^\star) = d_y$. \cite{dumont2022existence} and \cite{vayer2020contribution} showed that if $\mM^\star$ is full rank, then a Monge map for \ref{eq:GW-IP} is given by $T = - \nabla f \circ \mM^\star$, where $f$ is a Kantorovitch potential for $\cW^\varepsilon_{\ip{\cdot}{\cdot}}\br{M_\sharp \alpha, \beta}$. The entropic map for \ref{eq:linear_rot} verifies an analog result.
\begin{lemma}\label{prop:ent_map_sinkhorn_potentials}
    Let $(f_{\varepsilon, \varepsilon^\prime}, g_{\varepsilon, \varepsilon^\prime})$ be Sinkhorn potentials for $\cW^{\varepsilon^\prime}_{\ip{\cdot}{\cdot}}({\mM_\varepsilon^\star}_\sharp \alpha, \beta)$. Then, we can rewrite \eqref{eq:ent-map} as
    \begin{align}
        T_{\varepsilon, \varepsilon^\prime} = - \nabla f_{\varepsilon, \varepsilon^\prime} \circ \mM_\varepsilon^\star
    \end{align}
\end{lemma}
Using the previous lemma and results from \cite{pooladian2021entropic, zhang2022gromov}, we have the following convergence result.
\begin{proposition}\label{prop:full_rank_prop}
Let $\alpha, \beta \in \cP(\cX) \times \cP(\cY)$ be two measures with compact supports. Assume that they are regular enough (namely that they verify assumptions \textbf{(A1-3)} from \cite{pooladian2021entropic}). Let $\varepsilon, \varepsilon^\prime > 0$ and consider the map $T_{\varepsilon, \varepsilon^\prime}$ defined in \eqref{eq:ent-map}. Then, we have that $T_{\varepsilon, \varepsilon^\prime} \xrightarrow{L^2(\alpha)} T_{\varepsilon, 0}$ as $\varepsilon^\prime \to 0$, where $T_{\varepsilon, 0}$ is a Monge map for $\cW^0_{c_{\mM_\varepsilon^\star}}\br{\alpha, \beta}$, and $c_{\mM_\varepsilon^\star}(x, y) = - \ip{M_\varepsilon^\star x}{y}$. Moreover, we have along a subsequence that $T_{\varepsilon, 0} \xrightarrow{L^2(\alpha)} T_{0, 0}$ as $\varepsilon \to 0$, where $T_{0, 0}$ is a Monge map for \ref{eq:linear_rot}. Hence, along a subsequence,
\begin{align}
    \lim_{\varepsilon \to 0} \lim_{\varepsilon^\prime \to 0} T_{\varepsilon, \varepsilon^\prime} = T_{0, 0} \quad \text{in} \; L^2(\alpha).
\end{align}
\end{proposition}
\textbf{The case where $\mM^\star$ is not full rank.} The case where $\mM^\star$ is not full rank is more involved and with limited use in practice, as the theory in \cite{dumont2022existence} requires solving an OT problem between conditional probabilities, which are not accessible in practice. In the appendix, we simulate problems where we explicitly constrain $\mM^\star$ to being rank-deficient. We test on Gaussians (for which ground-truth OT maps are known \citep{salmona2022gromov}) and show that rank-defficiency has no effect on the convergence of the entropic map.

\textbf{Sparsifying transforms and entropic Monge maps.} As we saw in the previous section, we can show the convergence of the entropic map for \ref{eq:linear_rot} when $\mM^\star$ is full rank. With sparsifying and low-rank norms, it is unlikely that such a matrix is full rank. However, when using $\ell_{1, 2}$ regularization, our goal is to operate feature selection on $\cX$, so that we don't actually care about the dimensions on which the matrix $\mM_\varepsilon^\star$ is $0$ (which correspond to features we want to discard). Thus we can restrict our study to the existence of Monge maps between ${P_{I}}_\sharp \alpha$ and $\beta$ for the cost $c_{\mM^\star_I}(x, y) = - \ip{\mM_I^\star x}{y}$, where $I = \left\{i \in [d_x]: \norm{\mM_{;i}} \neq 0 \right\}$, $P_I$ is the projection operator on the dimensions $I$, and $\mM_I$ the restriction of $\mM$ to the columns indexed by $I$. Note, though, that this matrix could still be rank-deficient.

Similarly, nuclear norm regularization inherently makes $\mM^\star$ low-rank. As we saw earlier it results in a linear OT problem with an inner product cost between $\tilde{\mU}_\sharp \alpha$ and $\tilde{\mV}_\sharp \beta$ (i.e. $\cW_{\ip{\cdot}{\cdot}}(\tilde{U}_\sharp \alpha, \tilde{V}_\sharp \beta)$). If we are interested in low-dimensional representations of $\alpha$ and $\beta$, there always exists a Monge map between $\tilde{\mU}_\sharp \alpha$ and $\tilde{\mV}_\sharp \beta$, and the entropic map converges to the Monge map since it is associated with the inner-product cost on $\R^r$, where $r = \rk(\mM^\star)$ \citep{pooladian2021entropic}.

%% file: sections/applications.tex
\section{Applications} \label{sec:applications}
In all of the applications below, we use our proximal approach to compute transforms, jointly with the entropic map \eqref{eq:ent_map_practical} to displace points from one space to the other in and/or out of sample. That map's $\varepsilon$ regularization is selected using cross-validation on the training set using a grid in [5e-3, 1e-3, 5e-4, 1e-4, 5e-5, 1e-5].
\subsection{Multi-omics data Integration: Sparse and Low-Rank Transforms} \label{sec:sparse_lr_appli}

\textbf{Sparse Transforms.} We consider the scGM dataset \citep{cheow2016single} containing the gene expression and DNA methylation modalities for human somatic cells \citep{welch2017matcher}. The goal is to match the $177$ samples of the dataset across modalities using an entropic map from the gene expression ($d_x = 34$) to the DNA methylation domain ($d_y = 27$). Performance is measured in terms of Label Transfer Accuracy, a common way to evaluate single-cell data integration tasks \citep{demetci2022scot} (see Appendix).
\begin{figure}[t]
\captionsetup{font=footnotesize}
\centering
\begin{subfigure}{.265\textwidth}
 \centering
 \includegraphics[width=.97\linewidth, keepaspectratio]{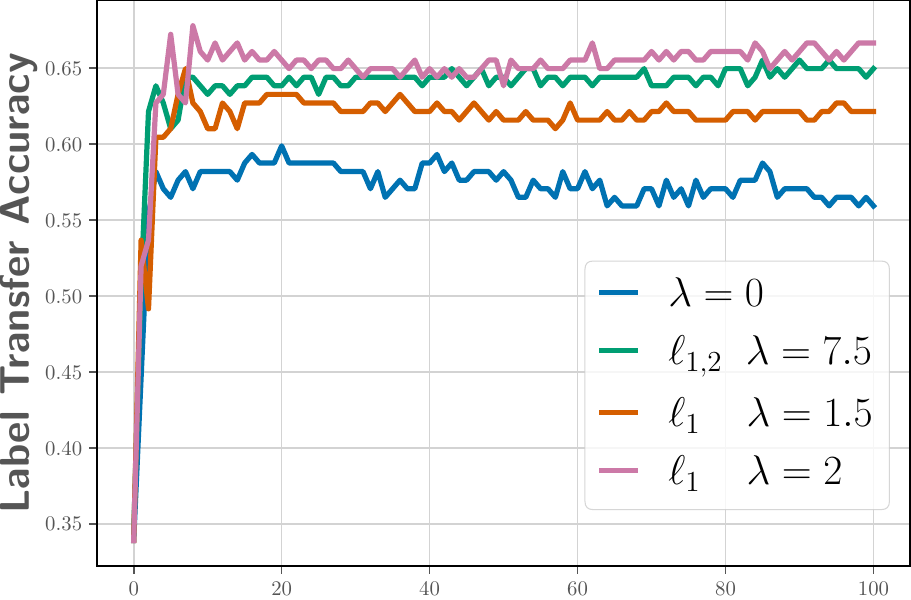}
 \label{fig:sub1}
\end{subfigure}%
\begin{subfigure}{.245\textwidth}
 \centering
 \includegraphics[width=.95\linewidth, keepaspectratio]{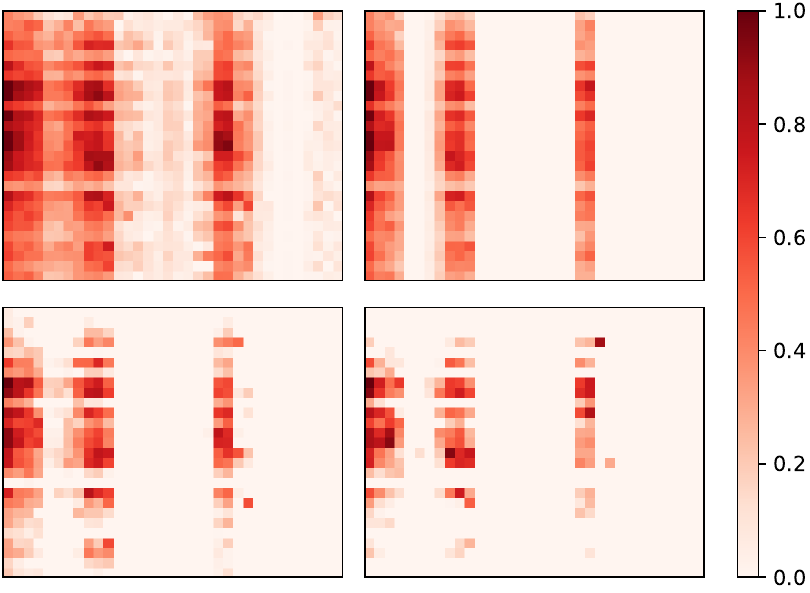}
 \label{fig:sub2}
\end{subfigure}
\vspace*{.1mm}
\caption{\textbf{Left:} Convergence of \ref{alg:Prox-ROT} with different sparsifying regularizations $\cR$ on the scGM dataset. X-axis: \# updates of the OT plan. Y-axis: LTA (higher is better). $\lambda = 0$ corresponds to standard \ref{eq:GW-IP}. $\varepsilon=0.05$. \textbf{Right:} $\mM_\varepsilon^\star$. \% zeros: \textit{Top} left $\lambda=0: 6\%$, right $\ell_{1,2}(\lambda=7.5): 67\%$; \textit{Bottom} left $\ell_{1}(\lambda=1.5): 71\%$, right $\ell_{1}(\lambda=2): 84\%$.}
\label{fig:sparse}
\end{figure}
We compare solving \ref{eq:GW-IP} using nested Sinkhorn iterations \citep{peyre2016gromov} to \ref{alg:Prox-ROT} with the $\ell_1$ and $\ell_{1,2}$ regularizations in \eqref{eq:reg-OT_g}. Adding a sparsity prior to the map $\mM$ results in a better performance (Figure \ref{fig:sparse}). Using an $\ell_1$ or $\ell_{1,2}$ regularization also outperforms all the methods in \cite[Figure 2]{demetci2022scot}.

\textbf{Low-Rank Transforms.} We consider the Neurips 2021 multimodal single-cell integration dataset~\citep{lance2022multimodal}. We use the {Site 1}/{Donor 1} stored in the MOSCOT package \citep{klein2023mapping}. Like scGM, the task consists in mapping $6,224$ cells from the chromatin accessibility $(d_y=8,000)$ to the gene expression domain $(d_x=2,000)$. 

We consider random subsets of the data of sizes $25, 50, 100, 250, 500, 1000, 6224$. On each subset, we run \ref{alg:Prox-ROT} with rank constraint between 5 and 12. We display for each subset the best-performing rank on average in Figure \ref{fig:lr_vs_sink} (that rank is always 11 or less). Performance is measured in terms of FOSCTTM (fraction of samples closer than the true match) depending on the subset size. Compared to GW, low-rank \ref{alg:Prox-ROT} handles settings where $n \ll d_y$ much better than GW. Note that \ref{alg:Prox-ROT}-$\rk$ still has benefits in the high $n$ regime since the dimension of the OT problem solved at each iteration is at most $12$ vs. $2,000$ for GW.

\vspace*{-1.5mm}
\subsection{Spatial Transcriptomics: Stochastic Fused GW-IP}
An important benefit of the cost-regularized formulation of GW-IP is that it lends itself ideally to stochastic optimization. Indeed, let $\alpha = \sum_{i=1}^m \alpha_i \delta_{x_i}$ and $\beta = \sum_{j=1}^n \beta_j \delta_{y_j}$ be two discrtete distributions. Then, using entropic regularization, we can rewrite \ref{eq:GW-IP-ROT} using the dual formulation of linear OT as
\begin{align}
    &\min_{\substack{\mM \in \R^{d_y \times d_x} \\ (f,g) \in \R^m \times \R^n}}  \sum_{i=1}^m f_i \alpha_i + \sum_{j=1}^n g_j \beta_j  + \frac{1}{8}\norm{\mM}_F^2 \label{eq:dual_fused} \\
    & - \varepsilon \sum_{i,j}\exp\br{\frac{f_i + g_j + \ip{\mM x_i}{y_j} + \eta\tilde{c}\br{\tilde{x}_i, \tilde{y}_j}}{\varepsilon}}\alpha_i\beta_j. 
\end{align}
where $\eta > 0$, and as in \eqref{eq:ROT_ent_fused}, we can use an additional known inter-space cost $\tilde{c}$, which corresponds to using Fused \ref{eq:GW-IP}. We can solve this problem using (epoch) stochastic gradient descent ascent. 
\begin{figure}[t]
\captionsetup{font=footnotesize}
\centering
\includegraphics[width=\linewidth, keepaspectratio]{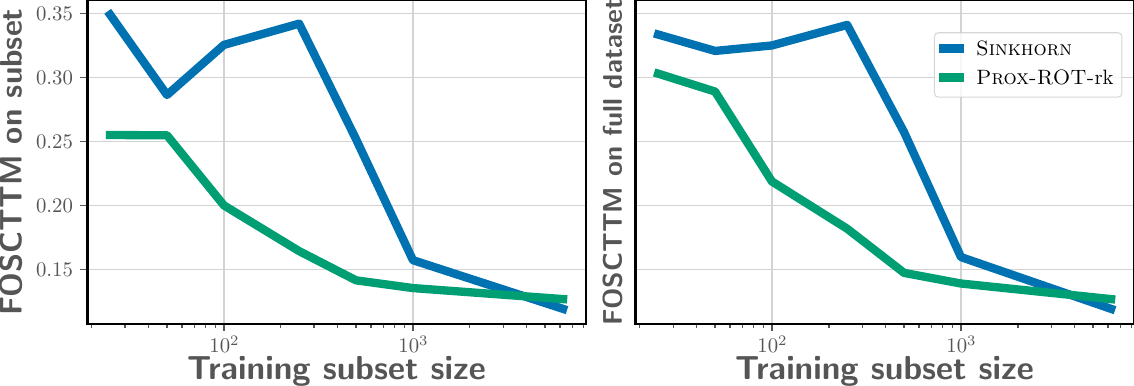}
\vspace*{.01mm}
\caption{Average performance (FOSCTTM - lower is better) across sampled subsets of GW vs. \ref{alg:Prox-ROT} depending on the size of training data. \textbf{Left:} Training performance on each subset size. \textbf{Right:} Evaluation on the full dataset of the map computed using a subset.}
\label{fig:lr_vs_sink}
\end{figure}

\begin{center}
\scalebox{0.76}{
\begin{tabular}{@{}lccccc@{}} \toprule
\textbf{Alg.} & \textbf{val} $\rho$ & \textbf{test} $\rho$ & \textbf{F1 macro} & \textbf{F1 micro}  & \textbf{F1 weighted} \\ \midrule
LFGW  & 0.365 & 0.443 & 0.576 & 0.720 & 0.714   \\
ULFGW  & \textbf{0.379} & \textbf{0.463} & 0.582 & 0.733 & 0.724   \\ \midrule
SFGW-IP  & 0.335 & 0.436 & \textbf{0.609} & \textbf{0.768} & \textbf{0.765}   \\ \bottomrule
\end{tabular}
}
\vspace*{.5mm}
\captionof{table}{Performance of different GW variants on a spatial transciptomics task. Details for LFGW and ULFGW can be found in \cite{scetbon2023unbalanced}. SFGW-IP refers to solving \eqref{eq:dual_fused} via epoch SGDA.}
\label{tab:spatial_tr}
\end{center}

We reproduce the experimental setting of \cite{scetbon2023unbalanced}. The goal is to align cells from two coronal sections of a mouse brain \citep{shi2023spatial}. The linear OT term in the fused formulation comes from a 30d PCA computed in gene expression space. We use validation Pearson correlation for hyperparameter selection. We show in Table \ref{tab:spatial_tr} that the stochastic approach, while having a lower per-iteration cost, is competitive with all GW variants, and significantly exceeds them in F1 scores (see appendix for more details).

\textbf{Conclusion.} We have leveraged a cost-regularized OT perspective on GW problems to propose new methods for OT across spaces that can induce structure (e.g. sparsity) on the transformation from one space to another. Our injection of sparsity into the transform $\mathbf{M}$ is unrelated to the sparsity observed naturally for couplings when solving the original OT problem, from LP duality or from other regularizations, as in e.g.\citep{pmlr-v84-blondel18a,liu2022sparsity}, nor is it related to that obtained for same-space Monge displacements by the MBO estimator~\citep{cuturi2023monge}. Our use of low-rank constraints for transform $\mathbf{M}$ is also unrelated to the low rank constraint on couplings introduced by~\citep{scetbon2021low,pmlr-v162-scetbon22b}.
The ability to add structural assumptions on across-space transforms opens up new perspectives to use transport \textit{across} high-dimensional modalities. We also showed the existence of Monge maps for our formulation, and demonstrated their applicability with entropic maps. We used a proximal alternated minimization algorithm together with structure-inducing regularizations, demonstrating applicability to single-cell multiomics data matching and spatial transcriptomics tasks.
\newpage

%% file: sections/appendix.tex
\aistatstitle{SUPPLEMENTARY MATERIAL \\ Structured Transforms Across Spaces \\ with Cost-Regularized Optimal Transport}

\section{Appendix for Section \ref{sec:background}}
\subsection{Proof of Proposition \ref{prop:Q_wasserstein_procrustes}}
Recall that the Wasserstein Procrustes problem \eqref{eq:wasserstein_procrustes} is
\begin{align}\label{eq:wasserstein_procrustes_app}
    \min_{\pi \in \Pi(\alpha, \beta)}\min_{\m{C} \in \cO_{d}} \int_{\cX \times \cY} \norm{\m{C}x - y}^2\diff \pi(x,y).
\end{align}
As noticed in e.g. Eq. 11 in \cite{zhang2017earth}, this formulation is an extension of the orthogonal Procrustes problem. Given two matrices $\mA, \mB \in \R^{d \times d}$ The orthogonal Procrustes problem aims to solve
\begin{align}
    \min_{\mC \in \cO_d} \norm{\mC A - B}_F.
\end{align}
The solution of this problem is $\mC^\star = \mU \mV^\top$, where $\mB \mA^\top = \mU \mSigma \mV^\top$. One can follow a similar procedure to show that a solution to the inner problem in \eqref{eq:wasserstein_procrustes_app} is given by $\mC^\top = \mU(\pi) \mV(\pi)^\top$, where $\int yx^\top \diff \pi(x,y) = \mU(\pi) \mSigma(\pi) \mV(\pi)^\top$ is an SVD. Replacing with this solution in the inner problem gives the desired \ref{eq:QOT} formulation:
\begin{align}
    \min_{\pi \in \Pi(\alpha, \beta)} \int_{\cX \times \cY} \norm{\mU(\pi) \mV(\pi)^\top x - y}^2\diff \pi(x,y).
\end{align}

\section{Appendix for Section \ref{sec:struct_ot}}
\subsection{Proof of Proposition \ref{prop:Q_linear_ROT}}
Starting from \ref{eq:linear_rot}, we have
\begin{gather}\label{eq:linear_rot_app}
    \min_{\substack{\pi \in \Pi(\alpha, \beta) \\ \mM \in \R^{d_y \times d_x}}} \int - \ip{\mM x}{y} d\pi + \cR(\mM) + \varepsilon \textsc{KL}(\pi || \alpha \otimes \beta), \\
    = \min_{\pi \in \Pi(\alpha, \beta)} \varepsilon \textsc{KL}(\pi || \alpha \otimes \beta) + \min_{\mM \in \R^{d_y \times d_x}}  \int - \ip{\mM x}{y} d\pi + \cR(\mM) \\
    \min_{\pi \in \Pi(\alpha, \beta)} \varepsilon \textsc{KL}(\pi || \alpha \otimes \beta) + \min_{\mM \in \R^{d_y \times d_x}}   - \ip{\mM }{\int yx^\top \diff \pi}_F  + \cR(\mM)\\
    = \min_{\pi \in \Pi(\alpha, \beta)} \varepsilon \textsc{KL}(\pi || \alpha \otimes \beta) - \max_{\mM \in \R^{d_y \times d_x}}   - \ip{\mM }{\int yx^\top \diff \pi}_F  + \cR(\mM) \\
    = \min_{\pi \in \Pi(\alpha, \beta)} \varepsilon \textsc{KL}(\pi || \alpha \otimes \beta) - \cR^*\br{\int_{\cX \times \cY} yx^\top\diff \pi(x,y)},
\end{gather}
which is the desired result.

\subsection{Proof of Proposition \ref{prop:convergence_alt_min}}
Consider \ref{eq:linear_rot}:
\begin{gather}
    \min_{\substack{\pi \in \Pi(\alpha, \beta) \\ \mM \in \R^{d_y \times d_x}}} \int - \ip{\mM x}{y} \diff \pi + \cR(\mM) + \varepsilon \textsc{KL}(\pi || \alpha \otimes \beta).
\end{gather}
Under the assumption of Proposition \ref{prop:convergence_alt_min}, the minimizer over $\pi$ is unique for any fixed $\mM$, and the minimizer over $\mM$ is unique for any fixed $\pi$. Thus, we can directly apply \cite[Prop. 2.7.1 and Ex. 2.7.1]{bertsekas1997nonlinear} to conclude.

\subsection{Proof of Lemma \ref{lem:prox_M}}
For any $\pi \in \Pi(\alpha, \beta)$,
\begin{gather}
        \mM(\pi) \triangleq \argmin_{\mM} - \int_{\cX \times \cY}\ip{\mM x}{y}d\pi + \frac{\norm{\mM}_F^2}{2} + \lambda g(\mM) \\
        = \argmin_{\mM} - \ip{\mM}{\int_{\cX \times \cY} yx^\top \diff \pi}_F  + \frac{\norm{\mM}_F^2}{2} + \lambda g(\mM) \\
        = \argmin_{\mM} \frac{1}{2} \norm{\mM -\int_{\cX \times \cY} yx^\top \diff \pi}^2_F + \lambda g(\mM) \\
        = \prox_{\lambda g}\br{\int_{\cX \times \cY} yx^\top d\pi},
\end{gather}
where in the third line we introduced an additional term $\frac{1}{2}\norm{\int_{\cX \times \cY}yx^\top d\pi}_F^2$ which doesn't depend on $M$.

\subsection{Proof of Proposition \ref{prop:l1_l12}}
Recall the cost-regularized OT problem with $\varepsilon=0$
\begin{align}
    \argmin_{\substack{\pi \in \Pi(\alpha, \beta) \\ \mM \in \R^{d_y \times d_x}}} - & \int \ip{\mM x}{y}d\pi + \frac{\norm{\mM}_F^2}{2} + \lambda g(\mM) \label{eq:reg-OT_g_app}.
\end{align}

Denote by $h : \Pi(\alpha, \beta) \times \R^{d_y \times d_x} \to \R$ and $H : \Pi(\alpha, \beta) \to \R$:
\begin{gather}
    h(\pi, \mM) = - \int_{\cX \times \cY}\ip{\mM x}{y}d\pi + \frac{\norm{\mM}_F^2}{2} + \lambda g(\mM) \quad \text{and} \quad H(\pi) = \min_{\mM \in  \R^{d_y \times d_x}} h(\pi, \mM).
\end{gather}

As shown in Lemma \ref{lem:prox_M}, given $\pi$, the solution of \ref{eq:reg-OT_g_app} is
\begin{align}
        & \mM(\pi) \triangleq \prox_{\lambda g}\br{\int_{\cX \times \cY} yx^\top d\pi},
\end{align}
Proposition \ref{prop:l1_l12} follows by evaluating the objective \eqref{eq:reg-OT_g_app} in the solution $\mM(\pi)$ using the corresponding proximal operator for each regularization.

\textbf{$\ell_1$-regularization.} Consider the case where $g(\mM) = \norm{\mM}_1$. Recall that for $\mM \in \R^{d_y \times d_x}$ and all $i,j \in [d_x] \times [d_y]$,
\begin{align}
    \prox_{\lambda \norm{\cdot}_1}\br{\mM}_{ij} & = \sgn(\mM_{ij})\br{\abs{\mM_{ij}} - \lambda}_+.
\end{align}
Thus, for $\mM(\pi)$,
\begin{align}\label{eq:prox_m(pi)_l1_app}
    \prox_{\lambda \norm{\cdot}_1}\br{\mM(\pi)}_{ij} & = \sgn\br{\int_{\cX \times \cY} x_i y_j d\pi}\br{\int_{\cX \times \cY} x_i y_j d\pi - \lambda}_+.
\end{align}
Define $I_{\pi} \triangleq \left\{(i, j) \in [d_x] \times [d_y] \,: \; \int_{\cX \times \cY} x_iy_j d\pi(x, y) > \lambda \right\}$. Plugging \eqref{eq:prox_m(pi)_l1_app} in \eqref{eq:reg-OT_g_app} gives, after reducing,
\begin{gather}
    \argmin_{\substack{\pi \in \Pi(\alpha, \beta)}} \sum_{(i,j) \in I_\pi}  \int_{\cX \times \cY} \int_{\cX \times \cY} - x_iy_jx_i^\prime y_j^\prime \diff \pi(x,y) \diff \pi(x^\prime, y^\prime) + \lambda \sum_{(i,j) \in I_\pi} \abs{\int_{\cX \times \cY} x_i y_j \diff \pi} - \frac{\lambda^2}{2}|I_{\pi}| \label{eq:inter_l1_math} \\
    = \int_{\br{\cX \times \cY}^2} \sum_{(i, j) \in I_{\pi}} \br{x_i x_i^\prime - y_j y_j^\prime}^2 \diff \pi \diff \pi - \sum_{(i,j)\in I_\pi}\br{\int_\cX x_i^2 \diff \alpha + \int_\cY y_j^2 \diff \beta  - \lambda \abs{\int_{\cX \times \cY} y_j x_i d\pi(x, y)}} - \frac{\lambda^2}{2}|I_{\pi}|.
\end{gather}
Let $\br{\pi^\star, \mM(\pi^\star)}$ be a solution to \eqref{eq:reg-OT_g}. Then
\begin{align}
    \pi \in \argmin_{\pi \in \Pi(\alpha, \beta)} H(\pi) \iff \pi \in \argmin_{\pi \in \Pi(\alpha, \beta)} h(\pi, \mM(\pi^\star)).
\end{align}
Hence, using the same calculations as in \eqref{eq:inter_l1_math}, \eqref{eq:reg-OT_g_app} is equivalent to 
\begin{align}
    \min_{\pi \in \Pi(\alpha, \beta)} \int_{\br{\cX \times \cY}^2} \sum_{(i, j) \in I_{\pi^\star}} \br{x_i x_i^\prime - y_j y_j^\prime}^2 \diff \pi^\star \diff \pi - \sum_{(i,j)\in I_\pi^\star}\br{\int_\cX x_i^2 \diff \alpha + \int_\cY y_j^2 \diff \beta  - \lambda \abs{\int_{\cX \times \cY} y_j x_i d\pi^\star(x, y)}}.
\end{align}
Thus, ignoring the terms that don't depend on $\pi$, \eqref{eq:reg-OT_g_app} is equivalent to 
\begin{align}
    \min_{\pi \in \Pi(\alpha, \beta)} \int_{\br{\cX \times \cY}^2} \sum_{(i, j) \in I_{\pi^\star}} \br{x_i x_i^\prime - y_j y_j^\prime}^2 \diff \pi^\star \diff \pi,
\end{align}
which is the desired result.

\textbf{$\ell_{1, 2}$-regularization.} We proceed similarly. Recall that for $\mM \in \R^{d_y \times d_x}$ and all $i \in [d_x]$,
\begin{align}\label{eq:prox_m(pi)_l1_app}
    \prox_{\lambda \norm{\cdot}_{1,2}}\br{\mM}_{:i} & = \br{1 - \frac{\lambda}{\max\left\{\norm{\mM_{:i}}\, , \lambda \right\}}}\mM_{:i},
\end{align}
i.e.
\begin{align}
    \mM_{:i} = \left\{
    \begin{array}{ll}
        \br{1 - \frac{\lambda}{\mM_{:i}}}\mM_{:i} & \mbox{if } \norm{M_{:i}} > \lambda \\
        0 & \mbox{otherwise.}
    \end{array}
\right.
\end{align}
Let $\mA \in \R^{d_y \times d_x}$ and $\tilde{\mA} = \prox_{\lambda \norm{\cdot}_{1,2}}(\mA)$. Let $I \triangleq \left\{\norm{\mA_{:i}} > \lambda \right\}$. Then,
\begin{gather}
    - \ip{\mA}{\tilde{\mA}}_F + \frac{\norm{\tilde{\mA}}}{2} + \lambda \norm{\tilde{\mA}}_{1,2} \\
    = \lambda \sum_{i \in I} \norm{\mA_{:i}} - \lambda^2 + \sum_{i \in I} \sum_j \br{- \mA_{ij}^2 \br{1 - \frac{\lambda}{\norm{\mA_{:i}}}} + \frac{1}{2}\br{1 - \frac{\lambda}{\norm{\mA_{:i}}}}^2 \mA_{ij}^2}\\
    = \lambda \sum_{i \in I} \norm{\mA_{:i}} - \lambda^2 |I| - \frac{1}{2} \sum_{i \in I} \sum_j \br{\norm{\mA_{:i}} - \lambda^2} \frac{\mA_{ij}^2}{\norm{\mA_{:i}}} \\
    = \lambda \sum_{i \in I} \norm{\mA_{:i}} - \lambda^2 |I| - \frac{1}{2} \sum_{i \in I} \br{\norm{\mA_{:i}}^2 - \lambda^2} \\
    = - \frac{1}{2} \sum_{i \in I} \norm{\mA_{:i}}^2 + \lambda \sum_{i \in I} \norm{\mA_{:i}} - \frac{\lambda^2}{2} |I|
\end{gather}
Thus, using $\mA = \mM(\pi)$ and $I_\pi \triangleq \left\{i \in [d_x] \,: \; \norm{\int_{\cY \times \cY}x_i y \diff \pi} > \lambda \right\}$, \eqref{eq:reg-OT_g_app} is equivalent to
\begin{align}
    \min_{\pi \in \Pi(\alpha, \beta)} \int_{\br{\cX \times \cY}^2}  \br{\ip{x_{I_{\pi}}}{x_{I_\pi}^\prime} -  \ip{y}{y^\prime}}^2 \diff \pi \diff \pi - \int_\cX \ip{x_{I_\pi}}{x_{I_\pi}^\prime}^2 \diff \alpha + \sum_{i \in I_\pi} \norm{\int_{\cX \times \cY} x_i y d\pi(x, y)} - \frac{\lambda^2}{2} |I_\pi|.
\end{align}

Proceding as we did for $\ell_1$ regularization, let $\br{\pi^\star, \mM(\pi^\star)}$ be a solution to \eqref{eq:reg-OT_g}. Then \eqref{eq:reg-OT_g_app} is equivalent to 
\begin{align}
    \min_{\pi \in \Pi(\alpha, \beta)} \int_{\br{\cX \times \cY}^2}  \int_{\cX \times \cY} \br{ \ip{x_{I_{\pi^\star}}}{x_{I_{\pi^\star}}^\prime} - \ip{y}{y^\prime}}^2 \diff \pi^\star \diff \pi - \int_\cX \ip{x_{I_{\pi^\star}}}{x_{I_{\pi^\star}}^\prime}^2 \diff \alpha + \lambda \sum_{i \in I_{\pi^\star}} \norm{\int_{\cX \times \cY} x_i y d\pi^\star(x, y)} - \frac{\lambda^2}{2} |I_{\pi^\star}|
\end{align}
Thus, ignoring the terms that don't depend on $\pi$, \eqref{eq:reg-OT_g_app} is equivalent to 
\begin{align}
        \min_{\pi \in \Pi(\alpha, \beta)} \int_{\cX \times \cY} \br{ \ip{x_{I_{\pi^\star}}}{x_{I_{\pi^\star}}^\prime} - \ip{y}{y^\prime}}^2 \diff \pi^\star \diff \pi,
\end{align}
which is the desired result.

\section{Appendix for Section \ref{sec:maps}}
\subsection{Proof for \ref{prop:monge_map_linear_costs}}
Before moving to the proof of Proposition \ref{prop:monge_map_linear_costs}, we first clarify a point concerning the existence and convergence of entropic maps for the inner-product linear OT. 

The results of \cite{pooladian2021entropic}, which define an entropic map and prove its convergence to the Monge map, are dedicated to the linear OT problem with the squared euclidean cost with $\cX = \cY \subset \R^d$:
\begin{align}
    \argmin_{\pi \in \Pi(\alpha, \beta)} \int_{\cX \times \cY} \frac{1}{2}\sqn{x - y} d\pi(x, y).
\end{align}
In our case, an intermediary result that we will need is the convergence of the entropic map for the inner-product cost $c(x, y) = - \ip{x}{y}$. However, by developing the square, one can see that the two problems are equivalent:
\begin{gather}
    \argmin_{\pi \in \Pi(\alpha, \beta)} \int_{\cX \times \cY} \frac{1}{2}\sqn{x - y} d\pi(x, y) \\
    = \argmin_{\pi \in \Pi(\alpha, \beta)} \int_{\cX \times \cY} \frac{1}{2}\sqn{x} \diff \pi(x,y) - \int_{\cX \times \cY} \ip{x}{y} \diff \pi(x,y)  + \frac{1}{2}\sqn{y} \diff \pi(x,y) \\
    = \argmin_{\pi \in \Pi(\alpha, \beta)} \int_{\cX} \frac{1}{2}\sqn{x} \diff \alpha(x) - \int_{\cX \times \cY} \ip{x}{y} \diff \pi(x,y)  + \frac{1}{2}\int_\cY \sqn{y} \diff \beta(y) \\
    = \argmin_{\pi \in \Pi(\alpha, \beta)} - \int_{\cX \times \cY} \ip{x}{y} \diff \pi(x,y).
\end{gather}
Thus, squared euclidean linear OT and inner-product OT have the same OT plan, i.e. (since the OT plans for the squared euclidean cost are induced by Monge maps) the same Monge maps. Hence, the results of \cite{pooladian2021entropic} directly extend to the inner-product cost. An additional, easy to verify fact, is the form of the Monge maps for the inner-product cost.

\begin{lemma}\label{lem:monge_map_ip}
    Let $f$ be a Kantorovitch potential for the inner product cost between two measures $\mu$ and $\nu$ with compact supports in $\R^d$. Then the Monge map for the inner-product cost (or equivalently for the squared euclidean cost) can be written as
    \begin{align}
    T(x) = - \nabla f(x).
\end{align}
\end{lemma}
\begin{proof}
    One can show that if $f$ is a Kantorovitch potential for the inner product cost, then $\tilde{f}: x \mapsto f(x) 
+ \frac{\sqn{x}}{2}$ is a Kantorovitch potential for the squared Euclidean cost. And we know from \cite{brenier1991polar} that given any Kantorovitch potential $\bar{f}$ for the squared euclidean cost, the unique (since Kantorovitch potentials are equal up to a constant) Monge map can be written as $T(x) = x - \nabla \bar{f}(x)$. Thus, given a Kantorovitch potential $f$ for the IP cost, we can write the Monge map for the inner product cost (or for squared euclidean costs since they have the same Monge map) as $T(x) = x - \nabla (f(x) + \tfrac{\sqn{x}}{2}) = - \nabla f(x)$.
\end{proof}

For completeness, we rewrite the entropic map of Definition \ref{def:ent_map}.
\begin{align}\label{eq:ent-map-app}
    T_{\varepsilon, \varepsilon^\prime}(x) = \frac{\int y \exp{\br{\br{g_{\varepsilon, \varepsilon^\prime}(y) + \ip{\mM_\varepsilon^\star x}{y}}/\varepsilon}}d\beta(y)}{\int \exp{\br{\br{g_{\varepsilon, \varepsilon^\star}(y) + \ip{\mM_\varepsilon^\star x}{y}}/\varepsilon}}d\beta(y)}.
\end{align}

\begin{proposition}\label{prop:ent_map_sinkhorn_potentials_app}
    Let $(f_{\varepsilon, \varepsilon^\prime}, g_{\varepsilon, \varepsilon^\prime})$ be Sinkhorn potentials for the inner product cost between ${\mM_\varepsilon^\star}_\sharp \alpha$ and $\beta$ with an $\varepsilon^\prime$ entropic regularization. Then,
    \begin{align}
        T_{\varepsilon, \varepsilon^\prime} = - \nabla f_{\varepsilon, \varepsilon^\prime} \circ \mM_\varepsilon^\star
    \end{align}
\end{proposition}
\begin{proof}
    Using optimality conditions for the $\varepsilon^\prime$-entropy regularized OT problem, we have that for all $\tilde{y} \in \cY$,
    \begin{align}
        f_{\varepsilon, \varepsilon^\prime}(\tilde{y}) = - \varepsilon \log\br{\int \exp{\br{\br{g_{\varepsilon, \varepsilon^\prime}(y) + \ip{\tilde{y}}{y}}/\varepsilon}}d\beta(y)}
    \end{align}
    Differentiating w.r.t. $\tilde{y}$ gives
    \begin{align}
        \nabla f_{\varepsilon, \varepsilon^\prime}(\tilde{y}) = - \frac{\int y \exp{\br{\br{g_{\varepsilon, \varepsilon^\prime}(y) + \ip{\tilde{y} }{y}}/\varepsilon}}d\beta(y)}{\int \exp{\br{\br{g_{\varepsilon, \varepsilon^\star}(y) + \ip{\tilde{y} }{y}}/\varepsilon}}d\beta(y)}.
    \end{align}
    Hence for all $x \in \cX$, evaluating at $\tilde{y} = \mM_\varepsilon^\star x$
    \begin{align}
        \nabla f_{\varepsilon, \varepsilon^\prime}(\mM_\varepsilon^\star x) = - \frac{\int y \exp{\br{\br{g_{\varepsilon, \varepsilon^\prime}(y) + \ip{\mM_\varepsilon^\star x }{y}}/\varepsilon}}d\beta(y)}{\int \exp{\br{\br{g_{\varepsilon, \varepsilon^\star}(y) + \ip{\mM_\varepsilon^\star x }{y}}/\varepsilon}}d\beta(y)} = - T_{\varepsilon, \varepsilon^\prime}(x).
    \end{align}
    That is, $T_{\varepsilon, \varepsilon^\prime} = - \nabla f_{\varepsilon, \varepsilon^\prime} \circ \mM_\varepsilon^\star$.
\end{proof}

Finally, before moving on to the proof, recall the following result from \cite{vayer2020contribution}, also appearing in \cite[Proposition 6]{dumont2022existence}.

\begin{proposition}[Theorem 4.2.3 in \cite{vayer2020contribution}, Informal.]
Finding a Monge map for the inner-product GW problem reduces to finding a Monge map between $\alpha$ and $\beta$ for the cost $c_\mM(x,y) = - \ip{\mM^\star x}{y}$, where $\mM^\star = \int yx^\top \diff \pi^\star$, and $\pi^\star$ is a solution to the inner-product GW problem. Such a map is given by 
\begin{align}
    T = - \nabla f \circ \mM^\star,
\end{align}
where $f$ is a Kantorovitch potential for $\cW_{0, \ip{\cdot}{\cdot}}\br{\mM_\sharp \alpha, \beta}$.
\end{proposition}

\begin{proof}[Proof of Proposition \ref{prop:full_rank_prop}]
Assume that $d_x \geq d_y$ and consider the case where $rk(\mM^*) = d_y$.

Using Lemma \ref{lem:monge_map_ip}, we now that the Monge map for $\cW^0_{\ip{\cdot}{\cdot}}\br{{\mM_\varepsilon^\star}_\sharp \alpha, \beta}$ is $T_{\varepsilon, 0}(x) = - \nabla f_{\varepsilon, 0}(x)$, where $f_{\varepsilon, 0}$ is a Kantorovitch potential for $\cW^0_{\ip{\cdot}{\cdot}}\br{{\mM_\varepsilon^\star}_\sharp \alpha, \beta}$. Since $f_{\varepsilon, \varepsilon^\prime}$ is a Sinkhorn potential for $\cW^{\varepsilon^\prime}_{\ip{\cdot}{\cdot}}\br{{\mM_\varepsilon^\star}_\sharp \alpha, \beta}$, we can apply
\cite[Corollary 1]{pooladian2021entropic} (here with the inner product cost rather than the squared euclidean cost), with $P = {\mM_\varepsilon}_\sharp \alpha$ and $Q = \beta$. We have that
\begin{gather}
    \int_{\cY} \sqn{\nabla f_{\varepsilon, \varepsilon^\prime}(y) - \nabla f_{\varepsilon, 0}(y)} \diff {\mM_\varepsilon}_\sharp \alpha(y) \leq {\varepsilon^\prime}^2 I_0({\mM_\varepsilon}_\sharp \alpha, \beta) + {\varepsilon^\prime}^{\br{\bar{\alpha} + 1}/2},
\end{gather}
Hence,
\begin{gather}
    \int_{\cX} \sqn{\nabla f_{\varepsilon, \varepsilon^\prime}(\mM_\varepsilon x) - \nabla f_{\varepsilon, 0}(\mM_\varepsilon x)} d\alpha(x)
    \leq  {\varepsilon^\prime}^2 I_0({\mM_\varepsilon}_\sharp \alpha, \beta) + {\varepsilon^\prime}^{\br{\bar{\alpha} + 1}/2},
\end{gather}
that is, using Lemma \ref{lem:monge_map_ip} and Proposition \ref{prop:ent_map_sinkhorn_potentials_app},
\begin{align}
    \int_{\cX} \sqn{T_{\varepsilon, \varepsilon^\prime}(x) - T_{\varepsilon, 0}(x)} \diff \alpha(x) \leq {\varepsilon^\prime}^2 I_0({\mM_\varepsilon}_\sharp \alpha, \beta) + {\varepsilon^\prime}^{\br{\bar{\alpha} + 1}/2}.
\end{align}
And since $I_0(\alpha_\varepsilon^\prime, \beta_\varepsilon^\prime) < \infty$ and doesn't depend on $\varepsilon^\prime$ \cite{chizat2020faster}, we have that for any $\varepsilon > 0$, $T_{\varepsilon, \varepsilon^\prime} \xrightarrow{L^2(\alpha)} T_{\varepsilon, 0}$ as $ \; \varepsilon^\prime \to 0$.

Now for the second statement. \cite{zhang2022gromov} showed that $\pi_\varepsilon^* \rightharpoonup_\varepsilon \pi_0^*$ along a subsequence. Hence, $\mM_\varepsilon := \int yx^\top \diff \pi_\varepsilon^\star \rightharpoonup_\varepsilon \mM := \int yx^\top d\pi^\star (x, y)$, which implies that ${\mM_\varepsilon}_\sharp \alpha \rightharpoonup \mM_\sharp \alpha$ along a subsequence.   

Since $\nabla f_{\varepsilon, 0}$ is a Monge map between ${\mM_\varepsilon}_\sharp \alpha$ and $\beta$, and $\nabla f_{0, 0}$ is a Monge map between $\mM_\sharp \alpha$ and $\beta$, it follows from the proof of \cite[Theorem 4.2]{philippis2013regularity} that $T_\varepsilon, 0 \xrightarrow[\varepsilon]{L_2(\alpha)} T_{0, 0}$. Indeed, it is shown that this implies that $\nabla f_{\varepsilon, 0}$ converges locally uniformly to $\nabla f_{0, 0}$, which in turn implies that $\nabla f_{\varepsilon, 0}$ converges uniformly to $\nabla f_{0, 0}$ since $\cX$ is compact. Thus, since $\mM_\varepsilon \rightarrow \mM$ and $\cX$ is compact, we have $T_{\varepsilon, 0} \xrightarrow[\varepsilon]{u} T_{0, 0}$, and consequently $T_{\varepsilon, 0} \xrightarrow[\varepsilon]{L_2(\alpha)} T_{0, 0}$ along a subsequence, which gives the desired result.
\end{proof}

\subsection{Using rank constraints in extremely high-dimensional problems} \label{app:rank_decomposition}
As we suggested at the end of Section \ref{sec:struct_ot}, using the proximal operators of the nuclear norm or of the rank can be infeasible in extremely high dimension because they involve computing the SVD of a high-dimensional matrix in the $\mM$-step of \ref{alg:Prox-ROT}. Thus, in such cases, we suggest to choose a rank $r \ll d_y$ and use the following regulariztion:
\begin{align}
    \cR(c) = \left\{
    \begin{array}{ll}
        \frac{1}{2} \norm{\mM_2^\top \mM_1}_F^2 & \mbox{if } \exists \br{\mM_1, \mM_2} \in \R^{d_x \times r} \times \R^{d_y \times r}: \forall (x, y), \; c(x,y) = \ip{\mM_1 x}{\mM_2 y} \\
        \infty & \mbox{otherwise.}
    \end{array}
\right.
\end{align}

In other words, we directly parametrize the matrix $\mM$ in \ref{eq:linear_rot} as a low-rank matrix: $\mM = \mM_2^\top \mM_1$. In this case, \ref{eq:linear_rot} rewrites:
\begin{gather}\label{eq:linear_rot_lr_app}
    \min_{\substack{\pi \in \Pi(\alpha, \beta) \\ \br{\mM_1, \mM_2} \in \R^{r \times d_x} \times \R^{r \times d_y}}} \int - \ip{\mM_1 x}{\mM_2 y} \diff \pi + \frac{1}{2} \norm{\mM_2^\top \mM_1}_F^2 + \varepsilon \textsc{KL}(\pi || \alpha \otimes \beta). \tag{L-$\cR \textsc{OT}_\varepsilon$}
\end{gather}
To solve this problem, we use alternated minimization (a.k.a. block coordinate descent) on each of the variables. Notably, we never have to compute the product $\mM_2^\top \mM_1$. Fixing $\pi$ and considering the first-order conditions on $(\mM_1, \mM_2)$ gives
\begin{align}
    & \br{\mM_2 \mM_2^\top}\mM_1 = \mM_2\int yx^\top \diff \pi(x,y)\\
    & \br{\mM_1 \mM_1^\top}\mM_2 = \mM_1\int xy^\top \diff \pi(x,y).
\end{align}
Thus, assuming that $\br{\mM_2 \mM_2^\top}$ and $\br{\mM_1 \mM_1^\top}$ are invertible, 
\begin{align}
    & \mM_1 = \br{\mM_2 \mM_2^\top}^{-1} \int \br{\mM_2 y}x^\top \diff \pi(x,y)\\
    & \mM_2 = \br{\mM_1 \mM_1^\top}^{-1} \int \br{\mM_1 x}y^\top \diff \pi(x,y).
\end{align}
So in practice, given $\pi_k$, we successively update 
\begin{align}
    & \mM_1^{k+1} = \br{\mM_2^k {\mM_2^k}^\top}^{\dagger} \int \br{\mM_2^k y}x^\top \diff \pi_k(x,y)\\ \label{eq:alg_rank_constraint}
    & \mM_2^{k+1} = \br{\mM_1^{k+1} {\mM_1^{k+1}}^\top}^{\dagger} \int \br{\mM_1^{k+1} x}y^\top \diff \pi_k(x,y),
\end{align}
where $\br{\mA}^\dagger$ denotes the Moore pseudoinverse of $\mA$. We can see that here all the operations are linear in $d_x$ and $d_y$. Moreover, as we choose a can choose a small rank $r \ll d_y$, the $\cO(r^3)$ dependency of computing the pseudoinverse can be much smaller than $\cO(d_y)$.

\textbf{Application of the updates \eqref{eq:alg_rank_constraint}.} We test the updates of \eqref{eq:alg_rank_constraint} in place of the proximal rank operator we used in Section \ref{sec:applications}. These updates result in very similar performance in practice (Figure \ref{fig:lr_fast_vs_sink}). We display Figure \ref{fig:lr_vs_sink} here again for the convenience of the reader. 

\begin{figure}[H]
\captionsetup{font=footnotesize}
\centering
\includegraphics[width=.75\linewidth, keepaspectratio]{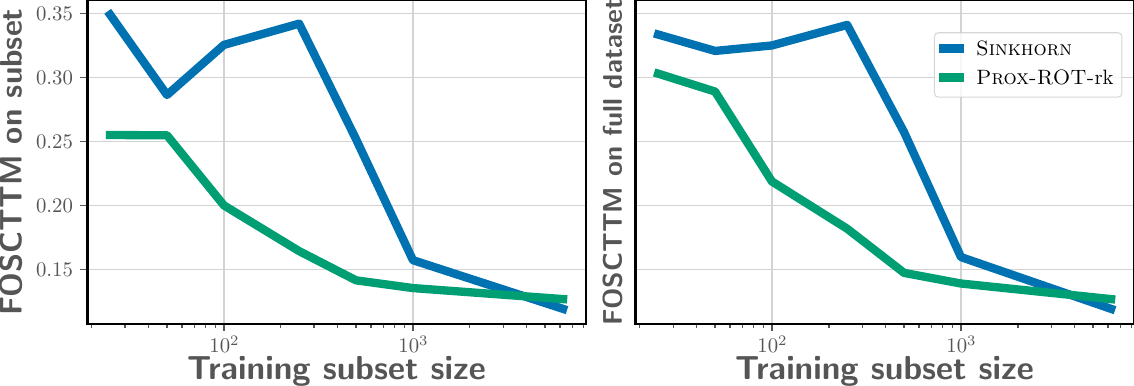}
\vspace*{.1mm}
\caption{Copy of Figure \ref{fig:lr_vs_sink}.}
\label{fig:lr_vs_sink_app}
\end{figure}

\begin{figure}[H]
\captionsetup{font=footnotesize}
\centering
\includegraphics[width=.75\linewidth, keepaspectratio]{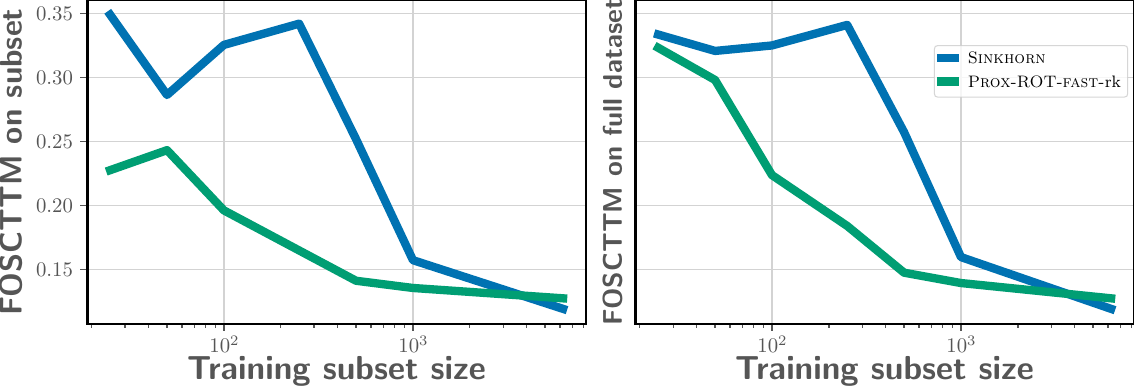}
\vspace*{.01mm}
\caption{Reproduction of Figure \ref{fig:lr_vs_sink_app} but with the updates \eqref{eq:alg_rank_constraint} instead of \ref{alg:Prox-ROT}.}
\label{fig:lr_fast_vs_sink}
\end{figure}

\section{Appendix for Section \ref{sec:applications}}

\subsection{Illustration of the action of the entropic map in the case where $\mM^\star$ is not full-rank}
One condition that we need for the convergence of the entropic map in Proposition \ref{prop:full_rank_prop} is that the matrix $\mM^\star$ (where $\br{\pi^\star, \mM^\star}$ is a solution of \ref{eq:linear_rot}) is full-rank. In Section \ref{sec:applications}, we showed through an experiment (see Figure \ref{fig:lr_vs_sink}) that even when the matrix $\mM^\star$ is explicitely constrained to being low-rank, the entropic map still results in good performance.

Here, we show through a more controlled experiment that indeed the fact that $\mM^\star$ is not full rank doesn't result in a divergent entropic map (as measured by the performance of the map when $\epsilon \approx 0$).

To do so, we simulate two point clouds coming from two Gaussian distributions, one in 20D and the other in 10D. We then use rank regularizations and take $\lambda = 12.5$, which is equivalent to constraining the matrix $\mM$ to being of at most of rank $5$. The reason we use Gaussians is that we know ground truth OT map between them \citep{salmona2022gromov}. 

\begin{figure}[H]
\captionsetup{font=footnotesize}
\centering
\includegraphics[width=\linewidth, keepaspectratio]{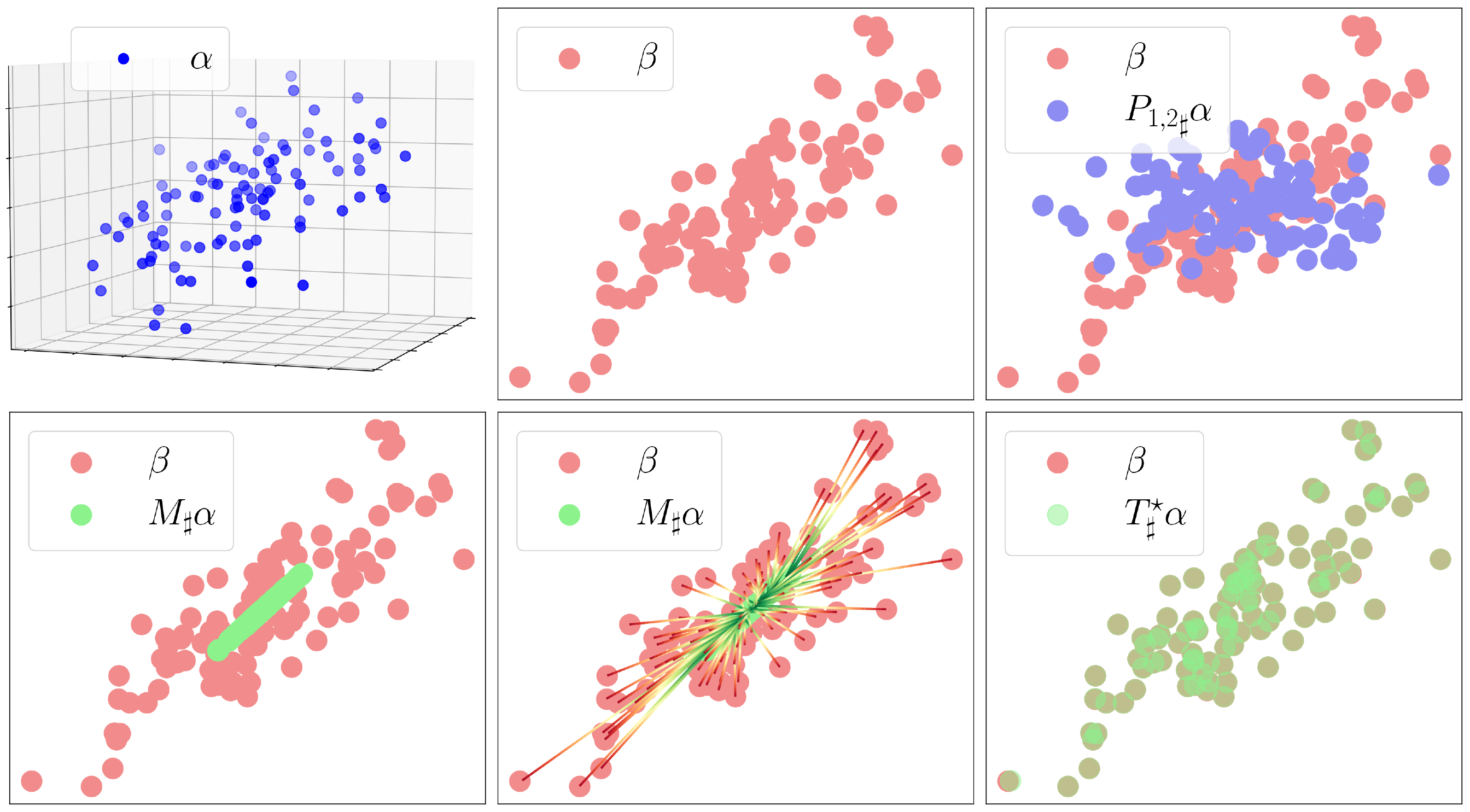}
\caption{Illustration of the action of the entropic map when $\mM^\star_\epsilon$ is low-rank.}
\label{fig:ent_map_low_rank}
\end{figure}

\subsection{Details on the evaluation metrics for sparse and low-rank transforms applications}
For our sparse transforms experiment on the scGM dataset (Section \ref{sec:sparse_lr_appli}), we measured the performance of the algorithms using Label Transfer Accuracy, as done in \cite{demetci2022scot}. Label Transfer Accuracy uses a k-NN classifier fitted on the data in the anchor domain (the one where we map using the entropic map, for us the data in the lower-dimensional data methylation domain $\left\{y_i\right\}_{i=1}^N$). Then the score is given as the test accuracy of the classifier on the mapped $\left\{T_\epsilon(x_i)\right\}_{i=1}^N$. 

For our low-rank transforms experiment on the Neurips 2021 multimodal single-cell integration dataset \citep{lance2022multimodal}, we measured the performance of the algorithms using the Fraction of Samples Closer than the True Match (FOSCTTM) as done in the \textsc{MOSCOT} package \citep{klein2023mapping}. The score computes the distances between each mapped sample $\left\{T_\epsilon(x_i)\right\}_{i=1}^N$ and the data $\left\{y_i\right\}_{i=1}^N$ and, given ground truth correspondences between data in $\cX$ and data in $\cY$, computes the FOSCTTM score as the average proportion of $y_i$'s closer to $T_\epsilon(x_i)$ than its true corresponding datapoint in $\left\{y_i\right\}_{i=1}^N$.

\subsection{Details on the spatial transcriptomics experiment}
As done in \cite{scetbon2023unbalanced}, to determine the best hyperparameters for the experiment, we ran a grid search and picked the best hyperparameters combination using performance on 10 validation genes and use pearson correlation as a validation metric. The hyperparameters we considered in our grid search are the entropic regularization $\epsilon$ and the fused cost parameter $\eta$ \eqref{eq:ROT_ent_fused}. We used a single minibatch size of $1000$ and a single number of ascent steps in $\br{f, g}$ of $1000$. For the SGD step in $\mM$, we use stochastic line search \citep{vaswani2019painless}. 